\documentclass[journal]{IEEEtran} 

\usepackage[linesnumbered,ruled,vlined]{algorithm2e}
\usepackage{cite}
\usepackage{amssymb}
\usepackage{float}
\usepackage[colorlinks,
            linkcolor=black,
            anchorcolor=black,
            citecolor=black]{hyperref}
\usepackage{makecell}
\usepackage{amsthm}
\usepackage{amsmath}
\newtheorem{proposition}{Proposition}

\ifCLASSINFOpdf
  \usepackage[pdftex]{graphicx}
\else
  \usepackage[dvips]{graphicx}
\fi
\usepackage{hyperref}
\usepackage{array}
\ifCLASSOPTIONcompsoc
 \usepackage[caption=false,font=normalsize,labelfont=sf,textfont=sf]{subfig}
\else
 \usepackage[caption=false,font=footnotesize]{subfig}
\fi 

\usepackage{url}

\begin{document}

\title{LinFormer: A Linear-based Lightweight Transformer Architecture For Time-Aware MIMO Channel Prediction}

\author{Yanliang Jin, Yifan Wu, Yuan Gao, \textit{Member, IEEE}, Shunqing Zhang, \textit{Senior Member, IEEE}, Shugong Xu, \textit{Fellow, IEEE}, Cheng-Xiang Wang, \textit{Fellow, IEEE}

\thanks{This paper is supported by the Innovation Program of Shanghai Municipal Science and Technology Commission under Grant 22511103202. (Yuan Gao is the corresponding author)} 
\thanks{Yanliang Jin, Yifan Wu, Yuan Gao, Shunqing Zhang and Shugong Xu are with the School of Communication and Information Engineering, Shanghai University, China, email: jinyanliang@staff.shu.edu.cn, 22721189@shu.edu.cn, gaoyuansie@shu.edu.cn, shunqing@shu.edu.cn and shugong@shu.edu.cn.}
\thanks{Cheng-Xiang Wang is with National Mobile Communications Research Laboratory, School of Information Science and Engineering, Southeast University, Nanjing, China, e-mail: chxwang@seu.edu.cn}
}

\maketitle
\begin{abstract}
The emergence of 6th generation (6G) mobile networks brings new challenges in supporting high-mobility communications, particularly in addressing the issue of channel aging. While existing channel prediction methods offer improved accuracy at the expense of increased computational complexity, limiting their practical application in mobile networks.
To address these challenges, we present LinFormer, an innovative channel prediction framework based on a scalable, all-linear, encoder-only Transformer model. Our approach, inspired by natural language processing (NLP) models such as BERT, adapts an encoder-only architecture specifically for channel prediction tasks. We propose replacing the computationally intensive attention mechanism commonly used in Transformers with a time-aware multi-layer perceptron (TMLP), significantly reducing computational demands. 
The inherent time awareness of TMLP module makes it particularly suitable for channel prediction tasks. We enhance LinFormer's training process by employing a weighted mean squared error loss (WMSELoss) function and data augmentation techniques, leveraging larger, readily available communication datasets. Our approach achieves a substantial reduction in computational complexity while maintaining high prediction accuracy, making it more suitable for deployment in cost-effective base stations (BS). Comprehensive experiments using both simulated and measured data demonstrate that LinFormer outperforms existing methods across various mobility scenarios, offering a promising solution for future wireless communication systems.
\end{abstract}


%
\IEEEpeerreviewmaketitle

\section{Introduction}
%
%
%
%
Multiple-input multiple-output (MIMO) technology plays a pivotal role in enhancing spectrum efficiency for fifth-generation (5G) and beyond 5G (B5G) mobile networks\cite{MMIMORobert2023, wang2023road, you2021towards}. Accurate channel state information (CSI) is crucial for BS to optimize beamforming and precoding, thereby maximizing data transmission rates in MIMO systems.

Conventional CSI acquisition methods, which rely on transmitting known pilot signals, face significant challenges in 5G and B5G networks. The increasing number of antennas in modern systems necessitates more pilot signals, leading to substantial overhead. Moreover, the demand for higher mobility causes rapid CSI variations, exacerbating the channel aging problem \cite{truong2013effects}. The use of outdated CSI in precoding can severely degrade system performance\cite{yin2020addressing, EricssonMIMO}.

Channel prediction has emerged as a promising solution to these challenges, allowing for the estimation of future CSI based on historical data without additional overhead. While traditional mathematical approaches such as linear extrapolation\cite{c2}, sum-of-sinusoids models\cite{c3}, Kalman filtering\cite{kim2020massive}, and autoregressive (AR) models\cite{c5} have been employed, they often struggle with the complexities of real-world scenarios, such as multi-path propagation and Doppler effects.

Recent advancements in artificial intelligence have led to the application of machine learning techniques in channel prediction\cite{huang2022artificial}. Deep learning approaches, including physics-inspired neural networks\cite{xiao2022c}, have shown promise in predicting static channel impulse responses. However, their performance in dynamic scenarios remains limited. To address this, some researchers have proposed hybrid approaches that combine model-driven and data-driven techniques in the angular delay domain, aiming to improve prediction robustness in high-mobility massive MIMO environments\cite{wu2021channel}.

In time-domain channel prediction, sequential frameworks such as Recurrent Neural Networks (RNNs)\cite{elman1990finding, c6}, Gated Recurrent Units (GRUs)\cite{GRU, stenhammar2024comparison}, and Long Short-Term Memory (LSTM) networks\cite{hochreiter1997long,c9,Jiang2} have demonstrated superior performance, even with limited data. These models benefit from their Markov inductive bias, which allows them to capture temporal dependencies effectively. However, they are susceptible to cumulative errors, particularly in long-term predictions. While RNNs inherently incorporate positional information, they struggle with long-term dependencies due to the vanishing gradient problem\cite{hochreiter1997long, GRU}. This limitation can lead to errors in multi-step predictions. GRUs have shown promise in tapped delay line (TDL) channel prediction, but their performance is often constrained by the size of available training datasets\cite{stenhammar2024comparison}.

Transformer models\cite{Transformer}, which have revolutionized NLP, have recently been applied to temporal channel prediction with promising results\cite{Accurate}. The attention mechanism employed by Transformers enables effective modeling of long-range dependencies and parallel sequence processing, potentially mitigating the cumulative error issues observed in RNN-based approaches.


Despite the advantages of the Transformer model, its complexity can lead to longer inference times for channel prediction, particularly when dealing with long input sequences and when base station devices have limited computing resources. 
Insufficient inference speed may render predicted channels applicable only for precoding the final few time intervals—or possibly none at all—as earlier intervals might elapse before the inference process concludes. 

Moreover, the extensive parameter volume in Transformer models elevates the risk of overfitting, potentially compromising their generalization capabilities. The attention mechanism, a key component of Transformers, is also vulnerable to the inherent noise and redundancy typically found in time series data \cite{eldele2024tslanet}.

The inherent permutation-invariance of the attention mechanism in Transformers poses an additional challenge. Even with position encoding, there remains a risk of losing vital temporal information. Recent research \cite{DLinear, RLinear} employing straightforward linear models has cast doubt on the efficacy of the original Transformer architecture for time series prediction. These studies highlight how the permutation-invariance of self-attention in Transformers may compromise the preservation of temporal information, which is crucial for tasks such as channel prediction.

Intriguingly, these studies demonstrated that a single linear layer can occasionally surpass sophisticated Transformer architectures in time series forecasting tasks. However, it's important to note that while these linear models may perform well with small, clean datasets, they often encounter difficulties when dealing with more complex and noisy temporal channel data.

Furthermore, the majority of existing models have been predominantly trained using simulated CSI data. This reliance on simulated data raises significant concerns regarding the practical applicability of these models in real-world mobile network environments.

\subsection{Our Contributions}
This paper addresses the aforementioned challenges in channel prediction by proposing a novel deep learning-based framework that achieves low prediction error, high inference speed, and improved generality in practical mobile networks. Our approach, named \textbf{LinFormer}, combines the strengths of \textbf{lin}ear layers and the Trans\textbf{former} encoder-only architecture. By replacing the attention layer with a linear-based layer while retaining the Transformer encoder's multi-block design, LinFormer offers a scalable solution for handling complex temporal channel series with enhanced inference speed and prediction accuracy. The key contributions of this work are as follows:

\begin{itemize}
\item We propose LinFormer, a scalable model structured as an encoder with trainable parameters exclusively derived from linear layers. This innovative design preserves the Transformer's architectural benefits while significantly reducing parameter count and overall complexity. We develop the TMLP module that directly models long-range dependencies within channel sequences by learning time-step-dependent weights. This approach is particularly effective for time-varying channels affected by Doppler frequency shifts and multipath propagation. 
\item We explore the potential of the proposed LinFormer comprehensively. Specifically, we investigate the impact of model parameter quantity and training data size on channel prediction error. Our analysis includes scaling the model by increasing the number of layers to expand parameter capacity, as well as examining the effects of varying training sample sizes and channel sequence lengths. Additionally, we enhance prediction accuracy through the implementation of WMSELoss and novel data augmentation techniques.
\item We perform extensive simulations to demonstrate LinFormer's superiority in channel prediction accuracy and inference speed across various scenarios, including different speeds and Signal-to-Noise Ratios (SNRs), using simulated CSI aligned with 3GPP standards. Our results show that LinFormer reduces channel prediction error by over 60\% compared to GRU at comparable inference speeds, while achieving marginally better prediction performance at six times the speed. Furthermore, we validate the generalization ability of the LinFormer in practical mobile networks using measured CSI data.
\end{itemize}

\subsection{Organizations and Notations}

The structure of this paper is as follows. Section II presents the system model for massive MIMO. Section III defines the channel prediction problem. Section IV introduces our proposed LinFormer model, including the rationale for using WMSELoss over traditional MSE, and a comprehensive description of the TMLP. Section V details our experimental results using simulated and measurement CSI data with in-depth discussion. Section VI offers concluding remarks and future directions.

Throughout this paper, we employ the following notation. Boldface lower-case letters represent column vectors, while boldface upper-case letters denote matrices. $(\cdot)^T$, $(\cdot)^H$, and $(\cdot)^{-1}$ signify the transpose, conjugate transpose, and inverse of a matrix, respectively. $\mathbb{C}$ and $\mathbb{R}$ represent the sets of complex and real numbers, respectively. $\mathbb{E}\{\cdot\}$ indicates the expectation operator. $||\cdot||_2$ denotes the Frobenius norm of a vector or matrix.

\section{System Model}

This section presents the system model for our wireless communication study.

Without loss of generality, we consider a multi-antenna wireless system with $T$ transmit and $R$ receive antennas operating in time division duplex (TDD) mode. In this setup, time is segmented into frames, which are further divided into slots. The BS acquires CSI via uplink pilot transmissions from the UE. The channel matrix at the BS, denoted as $\boldsymbol{H} \in \mathbb{C}^{R \times T}$, is a complex matrix with statistically independent coefficients:
\begin{equation}
  \boldsymbol{H} = \left(
    \begin{array}{ccc}
      h_{11} & \cdots & h_{1T}\\
      \vdots & \ddots & \vdots\\
      h_{R1} & \cdots & h_{RT}\\
    \end{array}
  \right),
\end{equation}
where $h_{ij}$ represents the channel coefficient between the $i$-th receive antenna and the $j$-th transmit antenna.

The BS employs beamforming to transmit signals to the UE. A symbol $c$ is weighted with a transmit vector $\boldsymbol{v}$ to form the transmitted signal vector $\boldsymbol{s}$. The received signal vector $\boldsymbol{y}$ is expressed as:

\begin{equation}
  \boldsymbol{y} = \boldsymbol{H}\boldsymbol{s} + \boldsymbol{n},
  \label{eq:received_signal}
\end{equation}
where the transmitted signals $\boldsymbol{s} = c\boldsymbol{v} = c\left(\begin{array}{ccc}v_1 & \cdots & v_T\end{array}\right)^T$. The additive white Gaussian noise (AWGN) vector $ \boldsymbol{n} = \left(\begin{array}{ccc}n_1 & \cdots & n_R\end{array}\right)^T$. The received signal $\boldsymbol{y}$ is combined by the receiving weighting vector $\boldsymbol{w} \in \mathbb{C}^{1 \times R}$ to obtain the estimated symbol $\hat{c}$:
\begin{equation}
  \hat{c} = \boldsymbol{w}\boldsymbol{y} = c\boldsymbol{w}\boldsymbol{H}\boldsymbol{v} + \boldsymbol{w}\boldsymbol{n}.
  \label{eq:estimated_signal}
\end{equation}

We primarily targets communication scenarios in 5G and beyond, where rapid user mobility leads to fast time-varying channels. Modern statistical wireless channel modeling stems from the Saleh-Valenzuela clustered channel model \cite{SVmodel}, which is particularly valuable for understanding how different factors, such as multipath propagation and Doppler shifts, can influence signal strength and transmission performance over time. It effectively simulates the random and unpredictable variations in channel quality that occur in real-world wireless communication environments.
The channel matrix for the $n$-th frame in TDD mode is calculated as
\begin{equation}
  \boldsymbol{H}^{(n)} = \sum_{l=1}^{L}{\alpha_l(n)e^{-j2{\pi}f_lnT_s}\boldsymbol{A}(\theta_l)\boldsymbol{A}^H(\phi_l)},
\end{equation}
where $\alpha_l(n)$ denotes the complex gain of the $l$-th propagation path, $T_s$ denotes the period of the frame, $f_l$ denotes Doppler shift, $\boldsymbol{A}(\theta_l)$ denotes the receive steering vector, $\boldsymbol{A}(\phi_l)$ is the transmit steering vector. $\theta_l$ and $\phi_l$ are the angles of departure (AoA) and angles of arrival (AoD), respectively. It is worth noting that the AoA and AoD are time-invariant, because the time scale of the frames in the 3GPP standard is in the order of tens of milliseconds \cite{3gpp.38.331, lim2021deep, liu2020robust}. Consequently, the time-varying nature of the channel is primarily attributed to the Doppler effect \cite{wang20206g, wang2022pervasive}. 

Time-domain channel prediction exploits the temporal correlation of time-varying channels. This correlation is quantified by the autocorrelation \cite{stuber2001principles} of $\boldsymbol{H}^{(t)}$, expressed as
\begin{equation}
    \boldsymbol{R}_{\boldsymbol{H}}(\tau) = \mathbb{E}[\boldsymbol{H}^{(n)}(\boldsymbol{H}^{(n+\tau)})^H] = \sigma^2J_0(2\pi f_d \tau T_s)\boldsymbol{A},
  \label{eq:autocorr}
\end{equation}
where 
\begin{equation}
\sigma^2 = \sum_{l}\mathbb{E}[\alpha_l^2],
\end{equation}
\begin{equation}
\boldsymbol{A}=\sum_{l}\boldsymbol{A}(\theta_l)\boldsymbol{A}^H(\phi_l)\boldsymbol{A}(\phi_l)\boldsymbol{A}^H(\theta_l),
\end{equation}
$f_d$ is the maximum Doppler frequency, and $J_0(\cdot)$ denotes the zero-order Bessel function of the first kind. Conventionally, it is assumed that $\sum_{l}\mathbb{E}[\alpha_l^2] = 1$ to ensure that the received signal power is equal to the transmitted signal power which results in $\boldsymbol{R}_{\boldsymbol{H}}(\tau)=J_0(2\pi f_d \tau T_s)\boldsymbol{A}$.


\section{Problem Formulation of Channel Prediction}

In this section, we formulate the channel prediction as a sequence-to-sequence problem.

The channels at the BS is time-varying. The channel over a period of time can be represented as a sequence of matrices as $\boldsymbol{H}_\mathrm{past} = \{\boldsymbol{H}^{(n-N_\mathrm{P}+1)}, \dots, \boldsymbol{H}^{(n)}\} \in \mathbb{C}^{N_\mathrm{P} \times R \times T}$,  where $N_\mathrm{P}$ is the number of past channels.  The BS needs to predict the future channels $\hat{\boldsymbol{H}}_\mathrm{future} = \{\boldsymbol{H}^{(n+1)}, \dots, \boldsymbol{H}^{(n+N_\mathrm{L})}\} \in \mathbb{C}^{N_\mathrm{L} \times R \times T}$, where $N_\mathrm{L}$ is the number of future channels. The goal of the channel prediction is to predict the future channel $\hat{\boldsymbol{H}}_\mathrm{future}$ based on the past estimated channel $\hat{\boldsymbol{H}}_\mathrm{past}$. 

According to Eq. (\ref{eq:received_signal}), the channel matrix $\hat{\boldsymbol{H}}_\mathrm{past}$ can be recovered from the received signal $\boldsymbol{y}$ and the transmitted signal $\boldsymbol{s}$. The BS can obtain the estimated past channel $\hat{\boldsymbol{H}}_\mathrm{past}$ by receiving the pilot signal from the UE by Minimum Mean-Square Error (MMSE) estimation \cite{mumtaz2016mmwave}.

First, we vectorize the channel matrix $\boldsymbol{H}$ to transmit pilot signal $\boldsymbol{p} \in \mathbb{C}^{RT \times RT}$ which is a diagonal matrix. 
 
The least square (LS) estimation $\hat{\boldsymbol{H}}_{LS}$ of the channel matrix is
\begin{equation}
  vec(\hat{\boldsymbol{H}}_\mathrm{LS}) = \boldsymbol{p}^{-1}\boldsymbol{y}.
\end{equation}
The MMSE estimation $\hat{\boldsymbol{H}}_\mathrm{MMSE}$ of the channel matrix is
\begin{equation}
  vec(\hat{\boldsymbol{H}}_\mathrm{MMSE}) = \boldsymbol{R}_{\boldsymbol{H}\boldsymbol{H}}(\boldsymbol{R}_{\boldsymbol{H}\boldsymbol{H}} + \frac{1}{\gamma_0}\boldsymbol{I}_{RT})^{-1}vec(\hat{\boldsymbol{H}}_\mathrm{LS}),
\end{equation}
where $\boldsymbol{R}_{\boldsymbol{H}\boldsymbol{H}} = \mathbb{E}\{\boldsymbol{H}\boldsymbol{H}^H\}$ is the autocorrelation matrix of the channel matrix $\boldsymbol{H}$, which can be obtained by statistics. $\gamma_0 = \frac{\sigma_p^2}{\sigma_n^2}$ is the SNR, $\sigma_p^2$ is the power of the pilot signal, $\sigma_n^2$ is the power of the noise.

Thus, the channel prediction problem is formulated as
\begin{equation}
  \hat{\boldsymbol{H}}_\mathrm{future} = f_\theta(\hat{\boldsymbol{H}}_\mathrm{past}),
\end{equation}
where $f_\theta(\cdot)$ is the channel prediction neural network, which is learned from the training data, $\theta$ is the learnable parameters of the neural network. 

The objective is to find $\theta$ such that the MSE between the predicted CSI and the ideal CSI is minimized.
\begin{equation}
    \theta = \arg \min_{\theta}(|f_\theta(\hat{\boldsymbol{H}}_\mathrm{past})-\boldsymbol{H}_\mathrm{future}|^2)
\end{equation}

The maximum ratio transmission (MRT) beamforming is adopted in the BS to maximize the SNR of the received signal. 

According to Eq. (\ref{eq:estimated_signal}), if $\boldsymbol{w}=\left(\begin{array}{ccc}\frac{1}{\sqrt{R}} & \cdots & \frac{1}{\sqrt{R}}\end{array}\right)$, the transmit weighting vector $\boldsymbol{v}$ can be obtained given the predicted future channel $\hat{\boldsymbol{H}}_\mathrm{future}$.
\begin{equation}
  \boldsymbol{v} = \frac{(\boldsymbol{w}\hat{\boldsymbol{H}}_\mathrm{future})^H}{||(\boldsymbol{w}\hat{\boldsymbol{H}}_\mathrm{future})^H||_2}.
\end{equation}

Therefore, the estimated signal vector $\hat{c}$ can be rewritten as
\begin{equation}
  \hat{c} = ca + \boldsymbol{w}\boldsymbol{n},
\end{equation}
where $a = \boldsymbol{w}\boldsymbol{H}_\mathrm{future}\boldsymbol{v} \in \mathbb{C}$.

So the channel capacity for the user in the downlink can be expressed as
\begin{equation}
  R = \log_2(1 + \gamma),
\end{equation}
where the overall SNR is written as
\begin{equation}
  \gamma = \lvert a\rvert^2\gamma_0.
\end{equation}

\begin{figure*}[htb]
  \centering
  \includegraphics[width=0.90\textwidth]{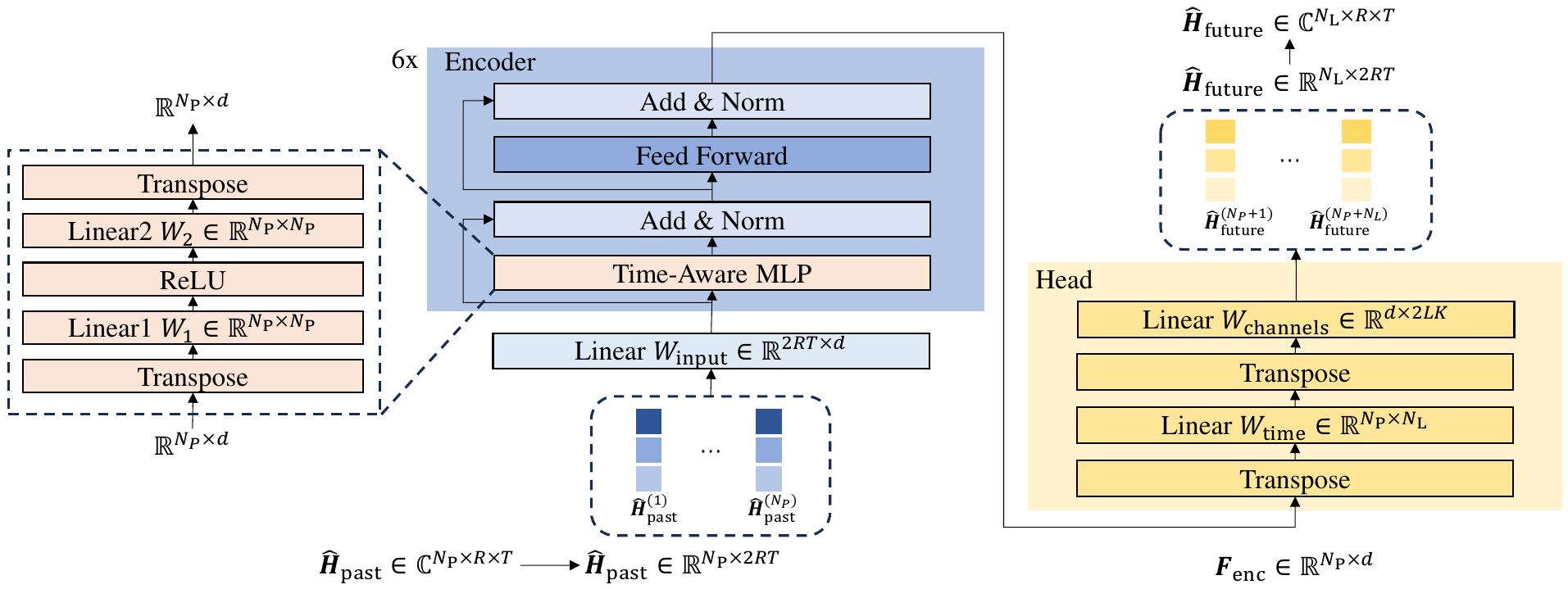}
  \caption{The Architecture of proposed LinFormer.}
  \label{fig:architecture}
\end{figure*}

\section{All-Linear Based Channel Prediction Framework}

LinFormer adapts the encoder-only model architecture,  which is inspired by the BERT model \cite{bert}. Here, the encoder maps an input embedding sequence of estimated past channels $\hat{\boldsymbol{H}}_{past} = (\hat{\boldsymbol{H}}_{past}^{(1)}, \hat{\boldsymbol{H}}_{past}^{(2)}, \cdots, \hat{\boldsymbol{H}}_{past}^{(N_\mathrm{P})})$ to a sequence of continuous representations $\boldsymbol{F} = (\boldsymbol{F}^{(1)}, \boldsymbol{F}^{(2)}, \cdots, \boldsymbol{F}^{(N_\mathrm{L})})$. Given the continuous representations $\boldsymbol{F}$, the prediction head then generates the future channel $\hat{\boldsymbol{H}}_{future} = (\hat{\boldsymbol{H}}_{future}^{(1)}, \hat{\boldsymbol{H}}_{future}^{(2)}, \cdots, \hat{\boldsymbol{H}}_{future}^{(N_\mathrm{L})})$ in parallel. 

\subsection{Overall architecture}

Our model integrates two novel components, i.e., the TMLP module and the dimension-wise separable linear head (DSLH), as depicted in Fig. \ref{fig:architecture}.
The TMLP is a lightweight module that replaces the self-attention mechanism in the original Transformer model, in which the weights are fixed time-step-dependent. The encoder consists of a stack of $N = 6$ identical layers, an architecture inherited from the original Transformer that makes LinFormer scalable. 
The DSLH is a simple linear layer that decomposes the weights of a fully connected matrix to reduce the number of parameters.


\begin{figure*}[htb]
  \centering
  \subfloat[Self-attention module in Transformer]{
      \label{fig:Attention}
      \includegraphics[width=0.3\textwidth]{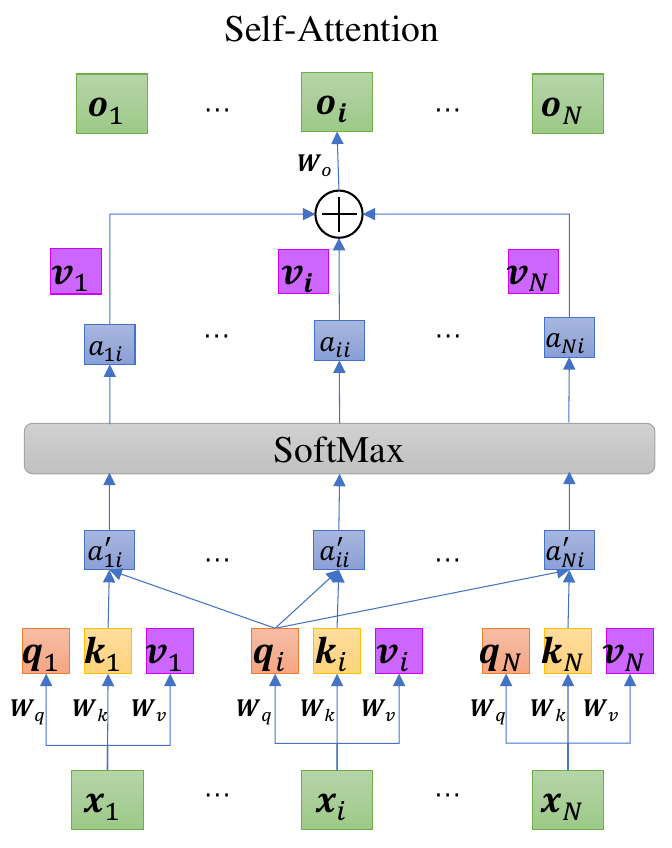}
    }
  \subfloat[TMLP in the proposed LinFormer]{
      \label{fig:TMLP}
      \includegraphics[width=0.3\textwidth]{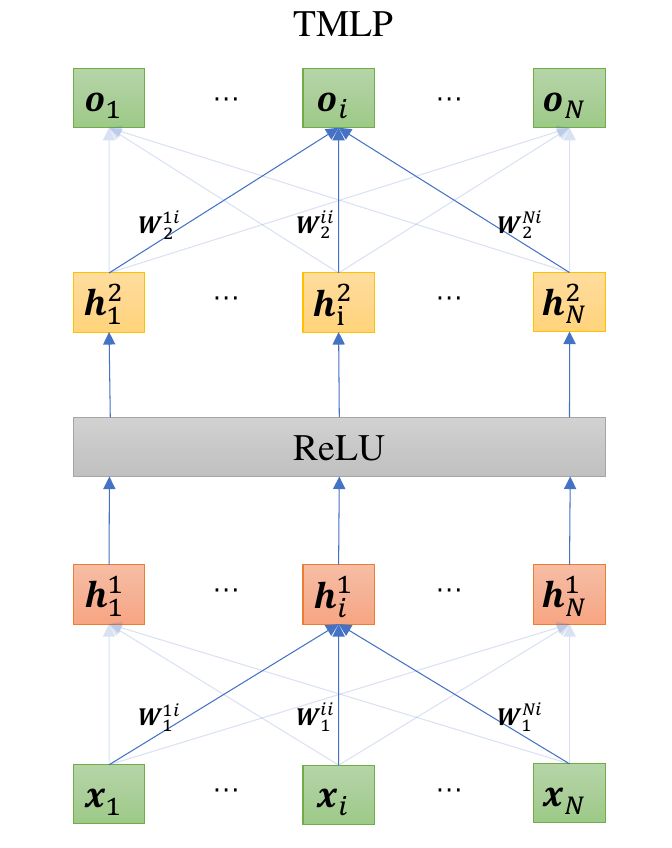}
    }
  \caption{Comparison of self-attention and proposed TMLP, taking the calculation of $\boldsymbol{O}_i$ as an example.}
  \label{fig:Attention_vs_TMLP}
\end{figure*}

\begin{algorithm}[htb]
  \SetAlgoLined
  \KwIn{$\boldsymbol{X}=[\boldsymbol{x}_1, \dots, \boldsymbol{x}_N]^T \in \mathbb{R}^{N \times d}$}
  \KwOut{$\boldsymbol{O}=[\boldsymbol{o}_1, \dots, \boldsymbol{o}_N]^T \in \mathbb{R}^{N \times d_v}$}
  Query matrix: $\boldsymbol{Q}=\boldsymbol{X}\boldsymbol{W}_q$\;
  Key matrix: $\boldsymbol{K}=\boldsymbol{X}\boldsymbol{W}_k$\;
  Value matrix: $\boldsymbol{V}=\boldsymbol{X}\boldsymbol{W}_v$\;
  $\boldsymbol{O}=\text{SoftMax}(\frac{\boldsymbol{Q}\boldsymbol{K}^T}{\sqrt{d_k}})\boldsymbol{V}$\;
  \caption{Self-Attention}
  \label{Algorithm:Standard}
\end{algorithm}

\subsection{Principles and Flaws of Attention}
Fig. \ref{fig:Attention} illustrates the self-attention layer, a key component of Transformer-based models. This layer transforms an input embedding matrix $\boldsymbol{X}$ into an output matrix $\boldsymbol{O}$. The output is essentially a weighted sum of value matrix $\boldsymbol{V}$, where the weights are determined by the scaled dot product of query and key vectors. 
The self-attention mechanism can be expressed in a general form as:

\begin{equation}
\boldsymbol{O} = \boldsymbol{A}\boldsymbol{X}\boldsymbol{B},
\label{eq:att}
\end{equation} where $\boldsymbol{X}$ is the input matrix and $\boldsymbol{B}= \boldsymbol{W}_v$. The attention matrix $\boldsymbol{A}$ is 
\begin{equation}
\boldsymbol{A} = f(\boldsymbol{X}) = \text{SoftMax}(\frac{\boldsymbol{X}\boldsymbol{W}_q(\boldsymbol{X}\boldsymbol{W}_k)^T}{\sqrt{d_k}}),
\end{equation} where $\boldsymbol{W}_q \in \mathbb{R}^{d \times d_k}$, $\boldsymbol{W}_k \in \mathbb{R}^{d \times d_k}$, and $\boldsymbol{W}_v \in \mathbb{R}^{d \times d_v}$ are learnable weight matrices. The dimensions $d$, $d_k$, and $d_v$ represent the dimensions of the input vector, key, and value, respectively. 
\begin{figure*}[htb]
  \centering
  \subfloat[The weights of the first linear layer in the first encoder layer.]{
      \label{fig:left}
      \includegraphics[width=0.4\textwidth]{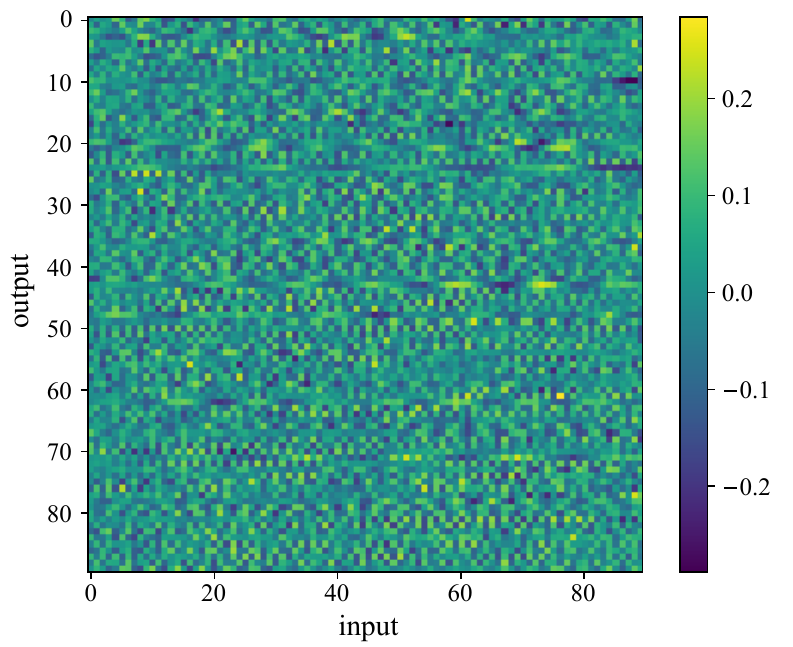}
    }
  \subfloat[The weights of the first linear layer in the last encoder layer.]{
      \label{fig:right}
      \includegraphics[width=0.4\textwidth]{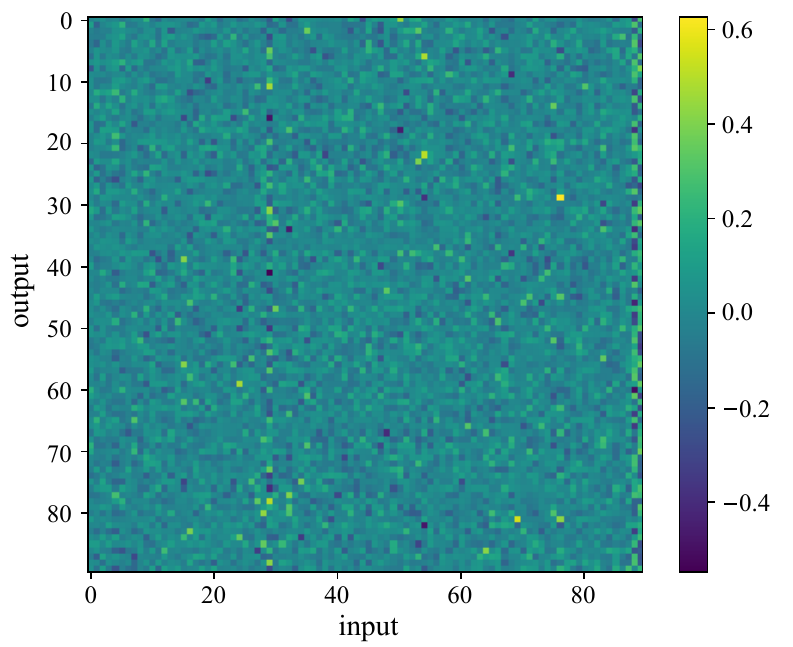}
    }
  \caption{The weights of TMLP.}
  \label{fig:Linear_weight}
\end{figure*}
The weight $\boldsymbol{A}$ used to weight the input $\boldsymbol{X}$ is mapped from $\boldsymbol{X}$ through the function $f(\cdot)$. So it is data-dependent.

While this mechanism is powerful in the field of NLP, it is important to note that its permutation-invariant nature can potentially compromise the preservation of temporal information \cite{Transformer}. In NLP tasks, each word in a sentence has rich semantic information.
Even when the word order is disrupted or the letters within a word are rearranged, humans or large language models (LLMs) can still comprehend the content \cite{cao2023unnatural}. 
This is because, the information conveyed by the text is of primary importance and relatively sparse, making the attention mechanism all that a Transformer needs. 

However, in the task of CSI prediction, the channel matrix at each time step does not carry strong semantics, and the information is evenly and densely distributed in a channel sequence. And temporal information becomes more critical. Simply using a Transformer may not yield satisfactory results. 
We derive the channel autocorrelation function as shown in Eq. (\ref{eq:autocorr}). Assuming the neural network can statistically obtain the autocorrelation function from a substantial number of training samples, the maximum Doppler frequency deviation $f_d$ can be determined given $\tau$. Once $f_d$ is identified, it becomes feasible to infer speed information, indicating that a channel sequence can represent the relative velocity between the receiver and transmitter. This capability exceeds what can be achieved by analyzing a single channel at an instantaneous time $t$.

\subsection{Time-Aware Multilayer Perceptron}
In order to make the model aware of time information, the weights in Eq. (\ref{eq:att}) are designed to be time-step-dependent. 
The TMLP module is obtained by applying MLP on the time dimension:
\begin{equation}
  \boldsymbol{O} = (\text{ReLU}(\boldsymbol{X}^T\boldsymbol{W}_1)\boldsymbol{W}_2)^T,
\end{equation}
where $\text{ReLU}(\cdot)$ is a element-wise activation function, $\boldsymbol{W}_1$ and $\boldsymbol{W}_2$ are the weight matrices. So the weights used to weight the input vector $\boldsymbol{X}$ is time-step-dependent. TMLP is relatively simple but very effective in making our model LinFormer time-step-dependent. 
The TMLP can directly model the long-range dependency of the sequence and avoid cumulative errors. And thanks to the multiplication operator, it is lightweight and computationally efficient, which is suitable for channel prediction tasks.

\begin{algorithm}[htb]
  \SetAlgoLined
  \KwIn{$\boldsymbol{X}=[\boldsymbol{x}_1, \dots, \boldsymbol{x}_N]^T \in \mathbb{R}^{N \times d}$}
  \KwOut{$\boldsymbol{O}=[\boldsymbol{o}_1, \dots, \boldsymbol{o}_N]^T \in \mathbb{R}^{N \times d_v}$}
  Linear 1: $\boldsymbol{X}_1 = \boldsymbol{X}^T\boldsymbol{W}_1$\;
  ReLU: $\boldsymbol{X}_2 = \text{ReLU}(\boldsymbol{X}_1)$\;
  Linear 2: $\boldsymbol{O} = (\boldsymbol{X}_2\boldsymbol{W}_2)^T$\;
  \caption{TMLP}
  \label{Algorithm:TMLP}
\end{algorithm}




An additional benefit of the TMLP module is its enhanced interpretability, as depicted in Fig. \ref{fig:Linear_weight}. 
The channel, assumed to be a stationary random process, exhibits temporal dependencies that can be captured through time-step-dependent weights learned during the training process. Fig. \ref{fig:left} illustrates the TMLP module in the first encoder layer. The presence of both low-frequency components (evident in rows 10, 20, 24, 42, 43, 48, 62, and others) and high-frequency components suggests that TMLP has successfully learned both global and local features of the input sequence. The low-frequency components correspond to global features, such as the main frequency of the sequence, while the high-frequency components represent local features, like specific inflection points. This dual representation aligns well with the inherent characteristics of wireless channels, which typically exhibit both global trends and localized variations.

In contrast, Fig. \ref{fig:right} depicts the TMLP module in the final encoder layer. Here, the weight patterns appear less structured and more uniform. This simplification is indicative of the deeper layer's role in abstracting and integrating the lower-level features learned in earlier layers. As the network progresses through its layers, it refines its representation, consolidating the diverse features into a more compact, high-level understanding of the channel characteristics.
Our motivation for utilizing TMLP is driven by two key factors: the significance of temporal information in channel sequence prediction and the need for real-time inference performance.

\subsection{Encoder}

The LinFormer architecture builds upon the multi-block design of the Transformer \cite{Transformer}, enhancing its scalability. At its core, the encoder comprises a stack of $N_{enc} = 6$ identical layers, each containing two sub-layers: a TMLP and a feed-forward network (FFN). To maintain the integrity of information flow, each sub-layer incorporates a residual connection, followed by layer normalization \cite{xiong2020layer}.

As illustrated in step two of $\boldsymbol{Algorithm\ \ref{Algorithm:encoder}}$, we propose to substitute the TMLP module for the multi-head self-attention mechanism in the encoder of Transformer. This substitution proves crucial, as our subsequent experiments demonstrate a significant improvement in channel prediction accuracy. The remaining components of the encoder layer remain consistent with the original Transformer design \cite{Transformer}.

The synergy between these architectural elements creates a robust framework that effectively balances the extraction of local and global temporal features while capturing cross-channel relationships. This unique combination positions the LinFormer as a powerful tool for temporal channel prediction analysis in wireless communication systems.

\begin{algorithm}[htb]
  \SetAlgoLined
  \KwIn{$\boldsymbol{F}_{0} \in \mathbb{R}^{N \times d}$}
  \KwOut{$\boldsymbol{F}_{N_{enc}} \in \mathbb{R}^{N \times d}$}
  \For{$i=1$ \KwTo $N_{enc}$}{
      TMLP: $f_1 = \text{TMLP}_i(\boldsymbol{F}_{i-1})$\;
      Add \& Norm: $\boldsymbol{F}_i = \text{Norm}(f_1+\boldsymbol{F}_{i-1})$\;
      Feed Forward: $f_2 = \text{FFN}_i(\boldsymbol{F}_i)$\;
      Add \& Norm: $\boldsymbol{F}_{i} = \text{Norm}(f_2+\boldsymbol{F}_i)$\;
    }
  \caption{Encoder}
  \label{Algorithm:encoder}
\end{algorithm}

\begin{figure}[htb]
  \centering
  \includegraphics[width=0.3\textwidth]{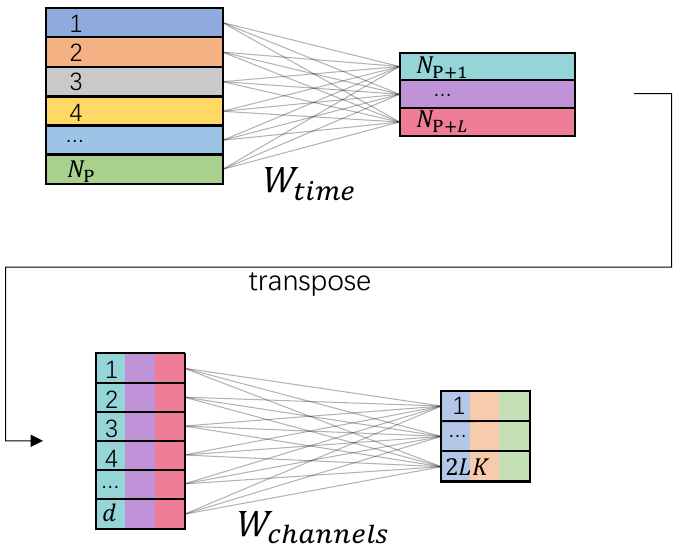}
  \caption{Dimension-wise separable linear head.}
  \label{fig:head}
\end{figure}

\subsection{Dimension-wise Separable Linear Head}
As shown in Fig. \ref{fig:head}, the prediction head consists of two linear layers.
Decomposing the weights of a fully connected matrix from $\boldsymbol{W}_\mathrm{FC} \in \mathbb{R}^{dN_\mathrm{P} \times 2N_\mathrm{L}RT}$ to $\boldsymbol{W}_\mathrm{time} \in \mathbb{R}^{N_\mathrm{P} \times N_\mathrm{L}}$ and $\boldsymbol{W}_\mathrm{channels} \in \mathbb{R}^{d \times 2RT}$ can significantly reduce the number of parameters.

Given the continuous representations $\boldsymbol{F} \in \mathbb{R}^{N_\mathrm{P} \times d}$, 
the DSLH can reduce the number of parameters to $N_\mathrm{P}N_\mathrm{L} + 2dRT$, which is much smaller than the fully connected matrix $\boldsymbol{W}_\mathrm{FC}$.
\begin{equation}
  \hat{\boldsymbol{H}}_\mathrm{future} = (\boldsymbol{F}^T\boldsymbol{W}_\mathrm{time})^T\boldsymbol{W}_\mathrm{channels},
\end{equation}
where $\boldsymbol{W}_\mathrm{time}$ is a fully connected matrix in the time dimension and $\boldsymbol{W}_\mathrm{channels}$ is a fully connected matrix in the channels dimension. 



\subsection{Loss Function}

To enhance the prediction accuracy of future channels, our approach incorporates a WMSELoss function. The MSE, a standard metric for regression tasks, quantifies the disparity between predicted and actual values, and is defined as
\begin{equation}
  MSE = \frac{1}{RTN_\mathrm{L}}\sum_{n=1}^{N_\mathrm{L}}||\boldsymbol{H}_\mathrm{future}^{(n)}-\hat{\boldsymbol{H}}_\mathrm{future}^{(n)}||_2^2,
\end{equation}
where $\boldsymbol{H}_\mathrm{future}$, $\hat{\boldsymbol{H}}_\mathrm{future}$ and $n$ denote the actual ideal future channels, the predicted future channels and the time step, respectively. The MSE weights errors at different time steps equally, which can be suboptimal for prediction tasks.

This limitation is addressed by considering the data processing inequality \cite{cover1999elements}, which suggests that data loses mutual information as it undergoes more processing. In the context of channel prediction, this implies that near-future channels are generally easier to predict due to their stronger correlation with historical channels compared to far-future channels.

To account for this temporal correlation, we derive weights for the proposed WMSELoss function based on the time-domain autocorrelation function of the channel information. We begin with the asymptotic expression of the zero-order Bessel function $J_0(n)$ for large $n$ \cite{abramowitz1968handbook}:

\begin{equation}
    J_0(n) \approx \sqrt{\frac{2}{\pi n}}\textit{cos}(n-\frac{\pi}{4}).
\end{equation}

Taking the absolute value and simplifying, We approximate the envelope by taking the absolute value and simplifying as:

\begin{equation}
    \textit{Envelope} \approx \sqrt{\frac{1}{n}}.
\end{equation}
Using the envelope of the autocorrelation function as the weight, the WMSELoss is defined as
\begin{equation}\label{WMSE_eq}
  WMSELoss = \frac{1}{RTN_\mathrm{L}}\sum_{n=1}^{N_\mathrm{L}}n^{-\frac{1}{2}}(\boldsymbol{H}_\mathrm{future}^{(n)}-\hat{\boldsymbol{H}}_\mathrm{future}^{(n)})^2.
\end{equation}
This formulation of WMSELoss, based on the time-domain autocorrelation of a single carrier and subchannel, allows for flexible extension to both frequency and spatial domains. As our current focus is on channel time series prediction for a single carrier, we will not compute the autocorrelation in frequency-domain and spatial-domain.


\subsection{Complexity Analysis}
This section presents a comparative analysis of the inference complexity between our proposed LinFormer model and the standard Transformer model. Our analysis primarily focuses on multiplication operations, as they are computationally more intensive than additions.

We begin by establishing that the complexity of multiplying matrices $\boldsymbol{A} \in \mathbb{R}^{a \times b}$ and $\boldsymbol{B} \in \mathbb{R}^{b \times c}$ is denoted as $abc$. This forms the foundation for our complexity estimations.

In the standard Transformer, the attention module's complexity is influenced by the input sequence length $N$, input sequence dimension $d$, and the dimensions of the key and value matrices ($d_k$ and $d_v$, respectively). Typically, $d_k$ and $d_v$ are set equal to $d$. Referring to $\boldsymbol{Algorithm\ \ref{Algorithm:Standard}}$, the complexity of matrix multiplications in steps 1-3 is $3Nd^2$. Step 4 involves two matrix multiplications with complexities $N^2d$ each. The multi-head attention mechanism introduces an additional output mapping $\boldsymbol{W}_o \in \mathbb{R}^{d \times d}$ \cite{Transformer}, contributing $Nd^2$ to the complexity. Thus, the total computational complexity of the standard self-attention module is $4Nd^2 + 2N^2d$.

The standard Transformer's decoder and output head contribute an additional complexity of $N_\mathrm{dec}(4N_\mathrm{L}d^2 + 2N_\mathrm{L}^2d + 2N_\mathrm{P}d^2 + 2N_\mathrm{L}d^2 + 2N_\mathrm{P}N_\mathrm{L}d + 8N_\mathrm{L}d^2) + 2N_\mathrm{L}dRT$, where $N_\mathrm{L}$ and $N_\mathrm{dec}$ represent the predicted sequence length and number of decoder layers, respectively.

In contrast, our proposed TMLP module, as outlined in $\boldsymbol{Algorithm\ \ref{Algorithm:TMLP}}$, has a computational complexity of $2N^2d$, derived from the matrix multiplications in steps 1 and 3.

We now present the computational complexity of our proposed LinFormer model:

\begin{proposition}\label{pro_cyclic}
The computational complexity of the proposed LinFormer is
\begin{equation}\label{complexity}
2N_\mathrm{enc}N_\mathrm{P}^2d + 2N_\mathrm{enc}N_\mathrm{P} d^2 + 2N_\mathrm{L}RTd + N_\mathrm{P}N_\mathrm{L}d,
\end{equation}
where $N_\mathrm{enc}$, $N_\mathrm{P}$, $N_\mathrm{L}$, and $d$ denote the number of encoder iterations, input sequence length, output sequence length, and model dimension, respectively.
\end{proposition}

\begin{proof}
The complexity of the feed-forward network remains $2Nd^2$, consistent with the Transformer. For an input sequence of length $N_\mathrm{P}$, the encoder's complexity is $2N_\mathrm{enc}N_\mathrm{P}^2d + 2N_\mathrm{enc}N_\mathrm{P}d^2$, where $N_{enc}$ represents the number of encoder iterations. The channel prediction head, following DSLH, contributes $N_\mathrm{P}N_\mathrm{L}d + 2N_\mathrm{L}RTd$ to the complexity. Summing these components yields the total computational complexity of LinFormer as expressed in Eq. (\ref{complexity}).
\end{proof}

This analysis demonstrates that the LinFormer model achieves significantly lower complexity compared to the standard Transformer. This reduction is achieved by replacing the attention module with the TMLP module and adopting an encoder-only architecture. The resulting low complexity makes the LinFormer model well-suited for practical wireless communication systems.

\section{Experiments and Discussions}


\subsection{Simulation Settings}
We consider a MIMO system with $R=2$ receive antennas and $T=4$ transmit antennas. The prediction model utilizes $N_\mathrm{P}$ frames of historical CSI $\hat{\boldsymbol{H}}_\mathrm{past}$ to predict $N_\mathrm{L}$ frames of future CSI $\hat{\boldsymbol{H}}_\mathrm{future}$. We examine various combinations of past and future frame numbers, with $N_\mathrm{P}$ taking values of 30, 60, or 90, and $N_\mathrm{L}$ being either 10 or 30. Consequently, each data sample spans a total of $N_\mathrm{P} + N_\mathrm{L}$ frames. In alignment with the 3GPP standard \cite{3gpp.38.331}, we maintain a consistent time interval of 0.625 ms between consecutive frames. This interval corresponds to the duration of a single Sounding Reference Signal (SRS) transmission, ensuring our simulation adheres to practical timing constraints in wireless communications.

\begin{figure}[htb]
  \centering
  \includegraphics[width=0.5\textwidth]{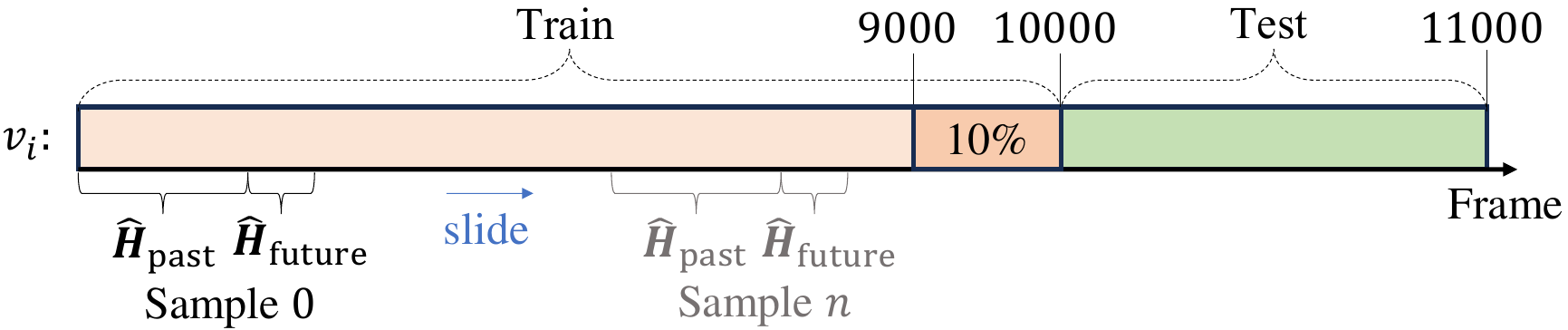}
  \caption{A sliding window used at each speed to obtain the training and test sets.}
  \label{fig:dataset}
\end{figure}

\begin{figure}[htb]
 \centering
 \includegraphics[width=0.35\textwidth]{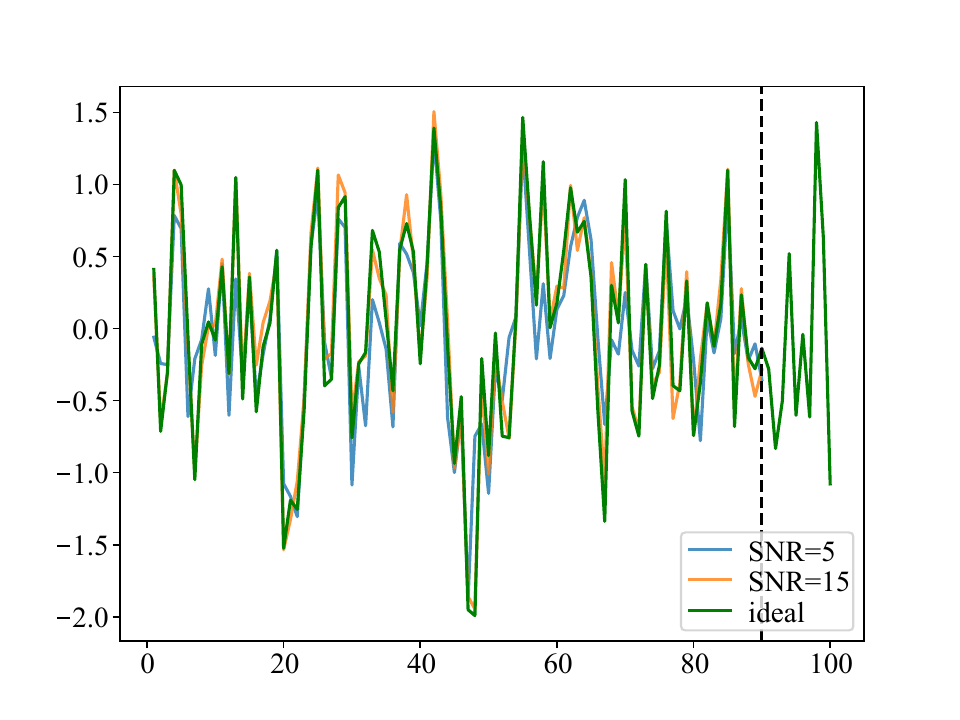}
 \caption{Data augmentation for improving prediction accuracy at low SNRs.}
 \label{fig:DA}
\end{figure}
\subsubsection{CSI Data}
To enhance the robustness of our proposed framework, we incorporate a diverse range of CSI parameters in the training process. These parameters include varying user speeds, delay spreads, and SNRs. Specifically, User speeds are uniformly distributed between $[v_\mathrm{min}, v_\mathrm{max}]$, while delay spreads, which account for multi-path effects, are randomly assigned values between 50 and 300 ns. The SNR, denoted as $\gamma$, is modeled as a uniform random variable ranging from 0 to 20 dB.

Standardized clustered delay line (CDL) model \cite{3gpp.38.901} for cellular communications is used in our simulation experiments. The CDL model complements the Saleh-Valenzuela model by providing further flexibility in modeling complex scenarios with different propagation characteristics. Although the Saleh-Valenzuela model is suitable for various environments, it may require parameter adjustments for different scenarios. The CDL model builds upon the Saleh-Valenzuela model by considering more complex multipath environments, specifying the delays, gains, and phases of paths through a delay line model. The CDL model offers different configurations (such as CDL-A, CDL-B, etc.) to accommodate various propagation environments, including urban, suburban, and indoor settings, making its application more flexible across different scenarios.
Fig. \ref{fig:dataset} illustrates our dataset generation process. For each randomly selected speed $v_i$ and delay spread, we generate 11,000 frames of ideal channel $\boldsymbol{H}_i \in \mathbb{C}^{11000 \times R \times T}$ using the 5G toolbox \cite{matlab5gtoolbox} in MATLAB, based on the CDL-B model. 
The first 10,000 frames and the remaining 1,000 frames constitute the training and test dataset, respectively. We have $10,001-(N_\mathrm{P} + N_\mathrm{L})$ training samples and $1,001-(N_\mathrm{P} + N_\mathrm{L})$ test samples. 

To further improve the model's performance across different SNR conditions, we employ a data augmentation technique, as depicted in Fig. \ref{fig:DA}. We introduce noise with random SNR to the historical CSI ${\boldsymbol{H}}_\mathrm{past}$, resulting in estimated past channels $\hat{\boldsymbol{H}}_\mathrm{past}$. These estimated channels serve as the input for our training samples. Importantly, while the input to the neural network is the re-estimated $\hat{\boldsymbol{H}}_\mathrm{past}$, the target output remains the noise-free future CSI $\boldsymbol{H}_\mathrm{future}$. This approach enables our model to predict accurate future channels across a spectrum of SNR conditions.

\subsubsection{Model Parameters and Comparison Models}
To ensure a fair comparison among the different models, we standardized certain architectural parameters. Specifically, we set the number of hidden layers in the GRU model, the number of encoder layers in the LinFormer model, and the number of encoder and decoder layers in the Transformer model to 6. For both the LinFormer and Transformer models, we set the output dimension $d$ of the encoder as 512.

This configuration results in distinct parameter counts for each model. The baseline Transformer model, with full attention mechanism, and the GRU contain approximately 30 M parameters. In contrast, the base LinFormer model, which employs a linear attention mechanism, has about 15 million parameters, reflecting its more compact architecture.
\subsubsection{Hardware Platform and Model Training}
We train the proposed and comparative models on a high-performance NVIDIA GeForce RTX 4090 GPU. We implemented the AdamW optimizer \cite{AdamW} with carefully tuned hyperparameters: a maximum learning rate of 0.0004, a batch size of 500, and a weight decay of 0.01. The training process spanned 100 epochs for all models. To optimize the learning process, we employed the 1cycle learning rate annealing policy \cite{OneCycleLR}, which dynamically adjusts the learning rate throughout the training.

For the testing phase, we transitioned to a more widely available GPU, the NVIDIA GeForce GTX 1050 Ti. This choice was motivated by its greater accessibility in practical applications, allowing us to evaluate our models under conditions more representative of real-world scenarios. This approach provides insights into the performance of each model on hardware commonly found in typical deployment environments.

\begin{figure}[htb]
  \centering
  \subfloat[The input sequence lengths of 30, 60 and 90 and output sequence length of 10.] {
\label{fig:len_vs_error}
\includegraphics[width=0.4\textwidth]{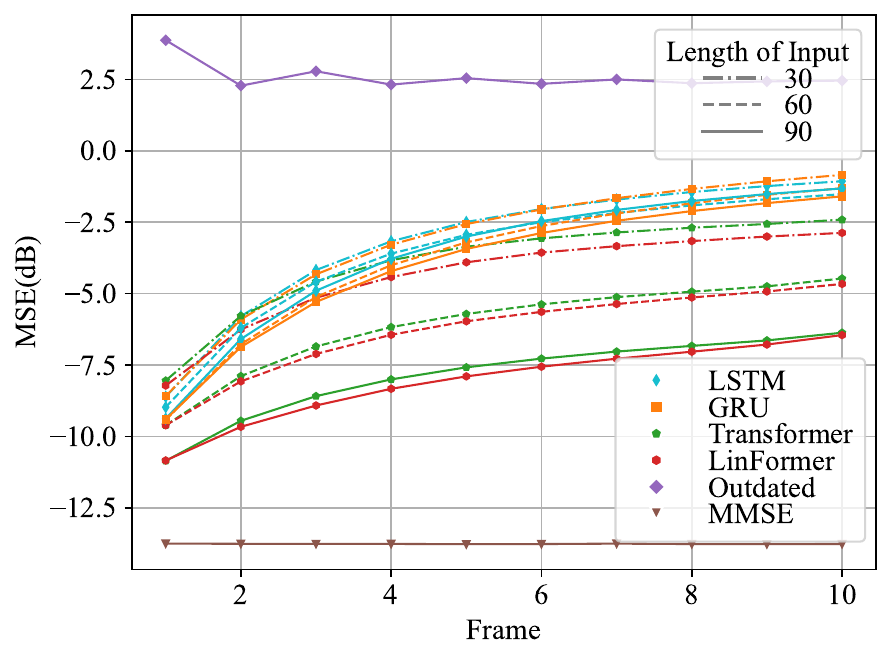} 
    }\
  \subfloat[The input sequence length of 90 and output sequence lengths of 10 and 30.] {
\label{fig:len_vs_error_10_30}
\includegraphics[width=0.4\textwidth]{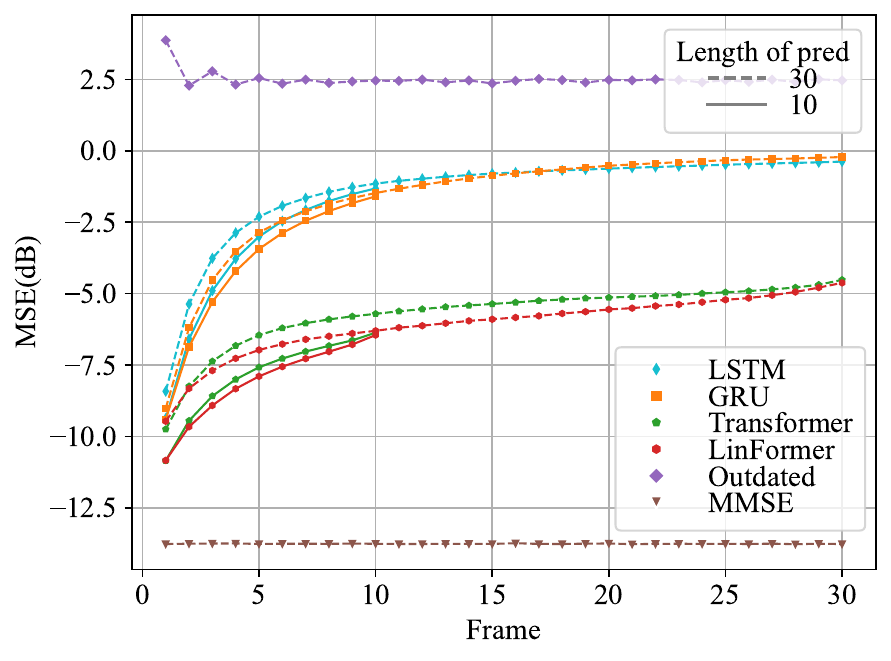} } 
\caption{Chanel prediction performance of different length of the input historical channel sequence and the predicted future channel sequence.}
\label{fig:len_vs_error2}
\end{figure}

\begin{figure}[htbp]
\centering\includegraphics[width=0.4\textwidth]{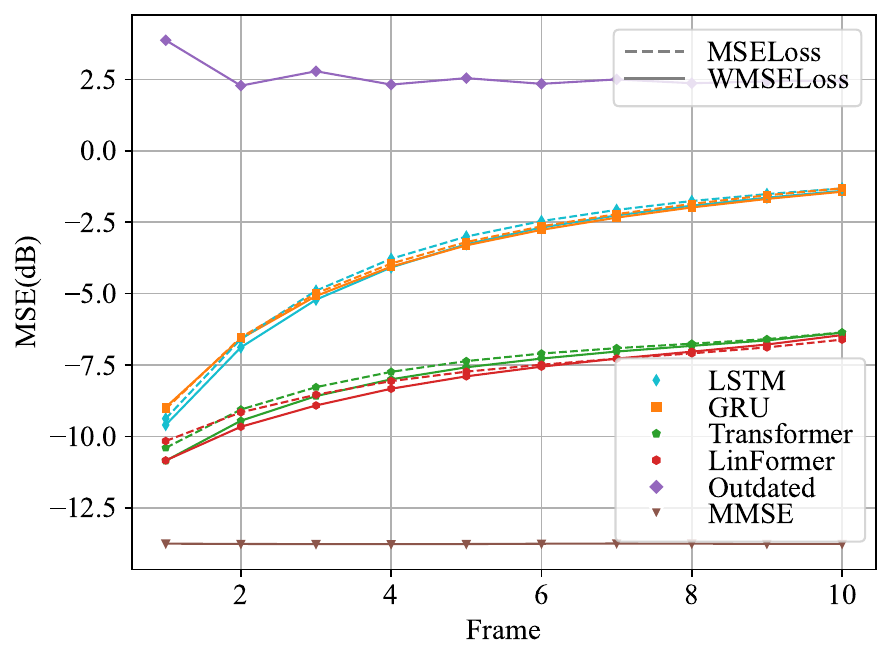}
\caption{MSE performance of different models using MSE loss and WMSE loss}
\label{fig:WMSE}
\end{figure}

\subsection{Performance Analysis Using Simulated CSI}
\subsubsection{The length of input and output}
Fig. \ref{fig:len_vs_error2} illustrates a comparative analysis of the proposed LinFormer, Transformer, GRU and LSTM models in terms of channel prediction MSE for each frame, considering various input and output sequence lengths. As illustrated in Fig. \ref{fig:len_vs_error}, both the Transformer and LinFormer models demonstrate significant performance improvements with longer input sequences. In contrast, the GRU and LSTM models show only marginal improvements under similar conditions. These results suggest that Transformer and LinFormer architectures are more capable at leveraging extended input sequences, effectively modeling long-term dependencies and extracting relevant features. Conversely, the inherent limitations of GRU and LSTM models in retaining long-term information hinder their ability to fully utilize extended input sequences.

Fig. \ref{fig:len_vs_error_10_30} reveals that LinFormer and Transformer models exhibit a slight advantage over GRU when using shorter output sequences. A notable finding is the substantial increase in channel prediction error for GRU and LSTM models between the first and tenth frame, attributable to cumulative error propagation. In contrast, LinFormer and Transformer models, leveraging parallel computing capabilities, demonstrate a considerably smaller increase in channel prediction error between the first and tenth frame.


\subsubsection{WMSELoss over Predicting Frames}
Analysis of Fig. \ref{fig:len_vs_error_10_30} reveals that the initial future frames exhibit the highest prediction accuracy, making them particularly valuable for base station operations such as precoding or beamforming. In light of this observation, we propose the implementation of WMSELoss function for model training. The resultant channel prediction error is depicted in Fig. \ref{fig:WMSE}. The results demonstrate notable performance improvements for both LinFormer and Transformer models between the first and fourth frames. In contrast, the GRU model and LSTM do not exhibit similar enhancements. Interestingly, the MSE performance for frames five through ten remains relatively stable, showing only minimal degradation. This stability can be attributed to the weighting scheme applied to the MSE at each time step, which is derived from the autocorrelation of the time channel sequence as defined in Eq. (\ref{WMSE_eq}).

\subsubsection{Robustness to Low SNR}
\begin{figure}[htb]
  \centering
  \includegraphics[width=0.4\textwidth]{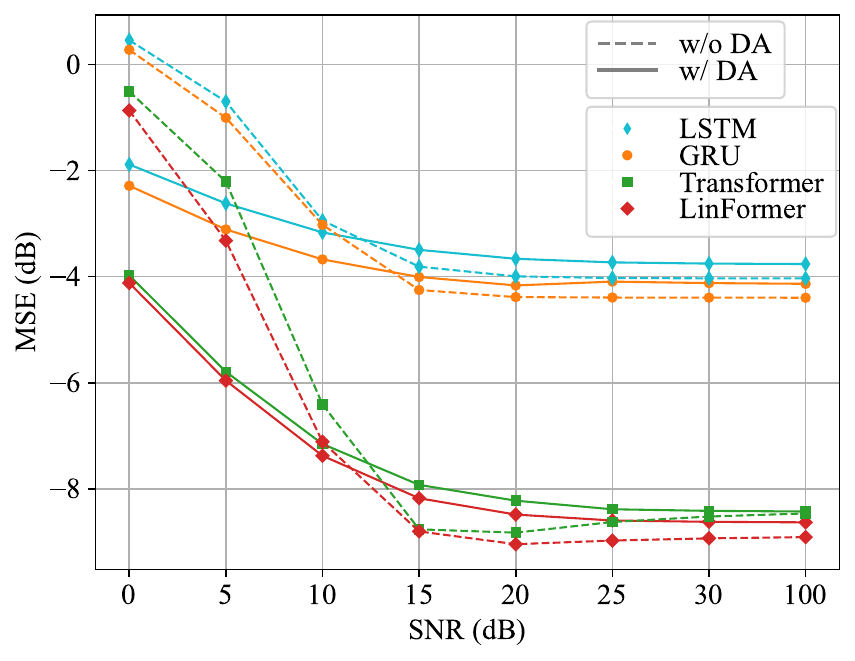}
  \caption{Robustness against different SNRs of different models with or without data augmentation. MSE averaged over all time steps and test samples.}
  \label{fig:SNR_vs_error}
\end{figure}
\begin{figure}[htbp]
\centering\includegraphics[width=0.4\textwidth]{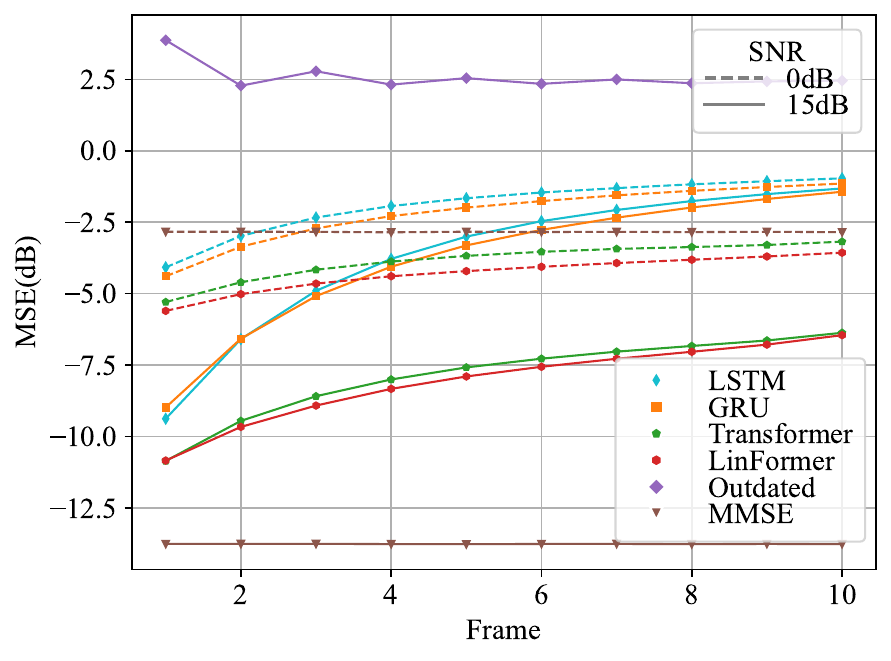}
\caption{MSE performance of different models under SNR$=0dB$ and $15dB$.}
\label{fig:t_MSEdB_SNR}
\end{figure}

To improve the generalization of our model under various SNRs, we propose a data augmentation methoddetailed in Section V.A. This method involves applying MMSE channel estimation to the input historical channels over an SNR range of 0 to 20 dB, yielding the model input $\hat{\boldsymbol{H}}_\mathrm{past}$.

Without data augmentation, all training samples are estimated at an SNR of 15 dB. Fig. \ref{fig:SNR_vs_error} illustrates that models employing data augmentation demonstrate enhanced MSE performance, particularly at low SNRs. Notably, the Transformer model without data augmentation achieves optimal MSE performance near 15 dB but exhibits performance degradation at higher SNRs, suggesting overfitting to the training conditions. Fig. \ref{fig:t_MSEdB_SNR} presents a comprehensive evaluation of all models at 0 dB and 15 dB, showcasing MSE performance across the entire prediction horizon. The results reveal that our proposed LinFormer model, when trained with data augmentation, achieves superior performance at SNR = 0 dB. It not only outperforms the Transformer model but also surpasses the immediate MMSE channel estimation for future time steps. This outcome underscores the generalization ability of LinFormer in low SNR environments.

\subsubsection{Model Analysis}
\begin{figure}[htb]
  \centering
  \includegraphics[width=0.4\textwidth]{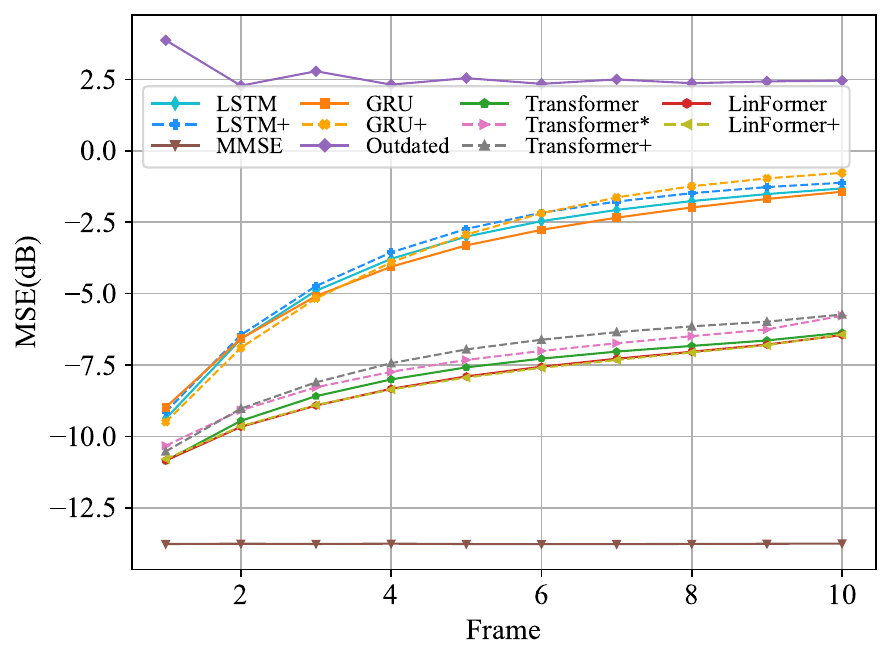}
  \caption{Module Analysis. '*' denotes the encoder-only Transformer architecture. '+' denotes the models trained on shuffled input sequence with a shared permutation.}
  \label{fig:model_analysis}
\end{figure}

\begin{figure}[htb]
  \centering
  \includegraphics[width=0.22\textwidth]{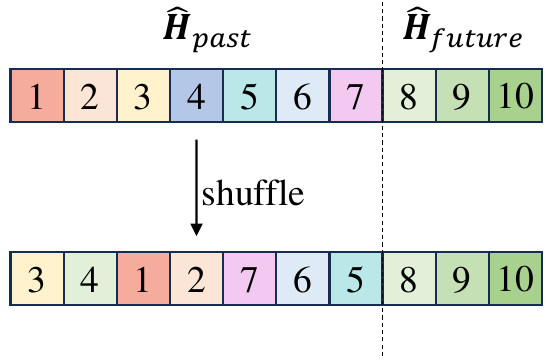}
  \caption{Shuffle all input sequences in the same permutation.}
  \label{fig:shuffle}
\end{figure}
Fig. \ref{fig:model_analysis} presents a comparative analysis between the TMLP module and the standard attention module. In this experiment, we replace the standard Transformer's decoder with the DSLH, as detailed in Section IV.E. The results indicate that the MSE performance deteriorates after the removal of the Transformer's decoder component. Nonetheless, within the encoder-only architecture, LinFormer demonstrates a significant improvement in performance by simply replacing the attention mechanism with the proposed TMLP module. To assess the sensitivity of the channel prediction models to temporal information in the input sequence, we conducted an experiment illustrated in Fig. \ref{fig:shuffle}. This involved shuffling the 90 time steps of each sample input using a consistent permutation across all samples. 

This result aligns with expectations: although the Transformer employs positional encoding, it still suffers from a loss of positional information. Consequently, the Transformer model faces challenges in accurately capturing the maximum Doppler frequency shift characteristics and differentiating between varying user terminal speeds, resulting in diminished prediction accuracy. Regarding RNN-like models, such as LSTM and GRU, their performance degradation can be attributed to the constraints of their hidden state size, which may lead to the loss of long-range dependencies in the input sequence.

As illustrated in Fig. \ref{fig:shuffle}, we shuffled the 90 time steps of each sample input according to the same permutation to investigate the model's sensitivity to the time information of the input sequence. 
The "+" symbol indicates that the model is trained on this shuffled input sequence across all samples. 

In contrast, our proposed LinFormer model, featuring time-step-dependent weights, demonstrates remarkable resilience, maintaining consistent performance even when subjected to identical input sequence shuffling.

\begin{figure}[htb]
  \centering
  \includegraphics[width=0.4\textwidth]{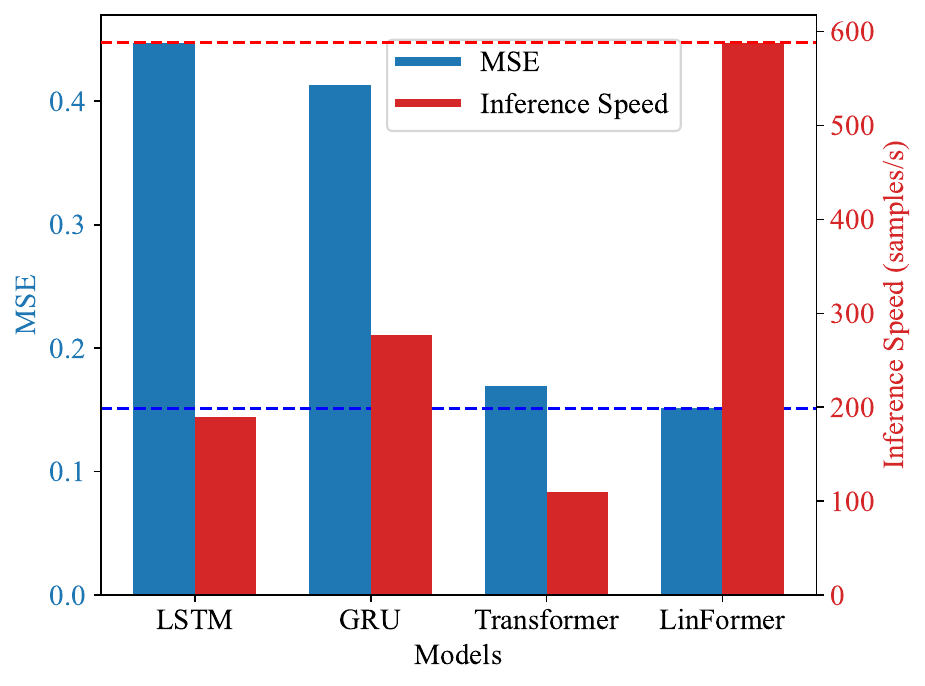}
  \caption{A comparison between LSTM, GRU, Transformer and LinFormer in terms of inference speed and MSE.}
  \label{fig:inference_speed}
\end{figure}

\subsubsection{Inference Speed}

Fig. \ref{fig:inference_speed} illustrated the inference speed and MSE of LSTM, GRU, Transformer and LinFormer in case of using 90 frames to predict 10 frames on a single NVIDIA GTX 1050 Ti GPU. 
The results demonstrate that LinFormer outperforms other models, achieving both the lowest MSE and the highest inference speed. The inference times for Transformer, LSTM, and GRU models range from approximately 4 to 10 ms. 
This duration is comparable to 10 SRS periods, as defined by industry standards. 
This means that these models are not fast enough to be applied in practical systems. In contrast, the proposed LinFormer model completes its predictions within 2 ms, demonstrating significant potential for practical applications in time-sensitive wireless communication scenarios.
\subsubsection{Scaling Efficiency}

\begin{figure}[htb]
  \centering
  \includegraphics[width=0.4\textwidth]{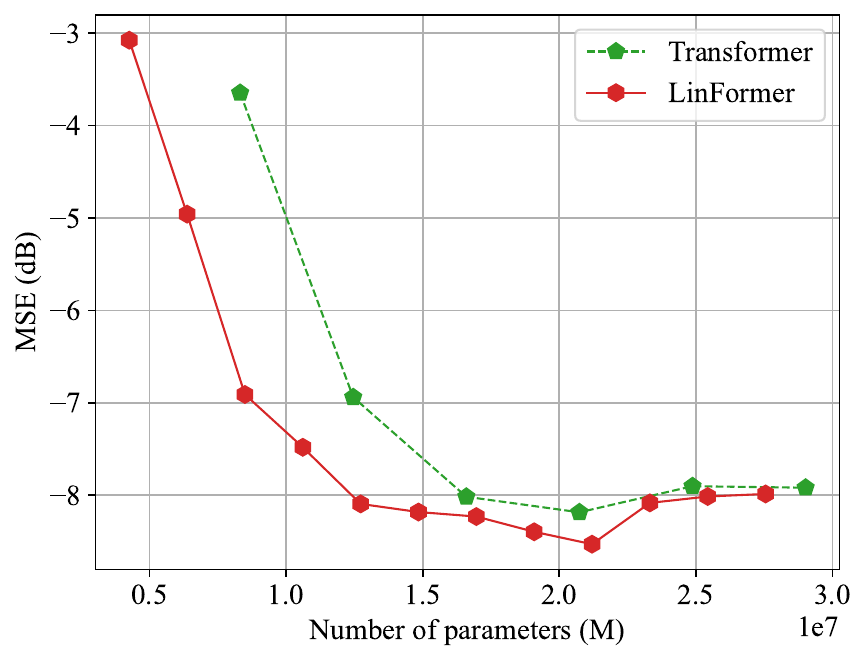}
  \caption{Scaling efficiency: the performance of LinFormer and Transformer with various model volumes. }
  \label{fig:params_vs_error}
\end{figure}

Fig. \ref{fig:params_vs_error} illustrates a comparative analysis of the MSE performance between the Transformer and LinFormer architectures across various parameter configurations. The MSE here is the average of all time steps and test samples.  We compared the performance of Transformer and LinFormer architectures with varying parameter volumes. Specifically, we increased the number of encoder layers in LinFormer to raise its parameter count, while for Transformer, we increased the number of both encoder and decoder layers. 

Interestingly, we observed that beyond a certain threshold of parameter count, the MSE performance on the test set began to deteriorate for both models. This phenomenon, indicative of overfitting, may be attributed to the limited size of our dataset. This finding underscores the importance of balancing model complexity with the available training data to achieve optimal performance in channel prediction tasks.

\subsubsection{Data Volume}
\begin{figure}[htb]
  \centering
  \subfloat[Channel prediction MSE performance]{
      \label{fig:t_MSEdB_data}
      \includegraphics[width=0.4\textwidth]{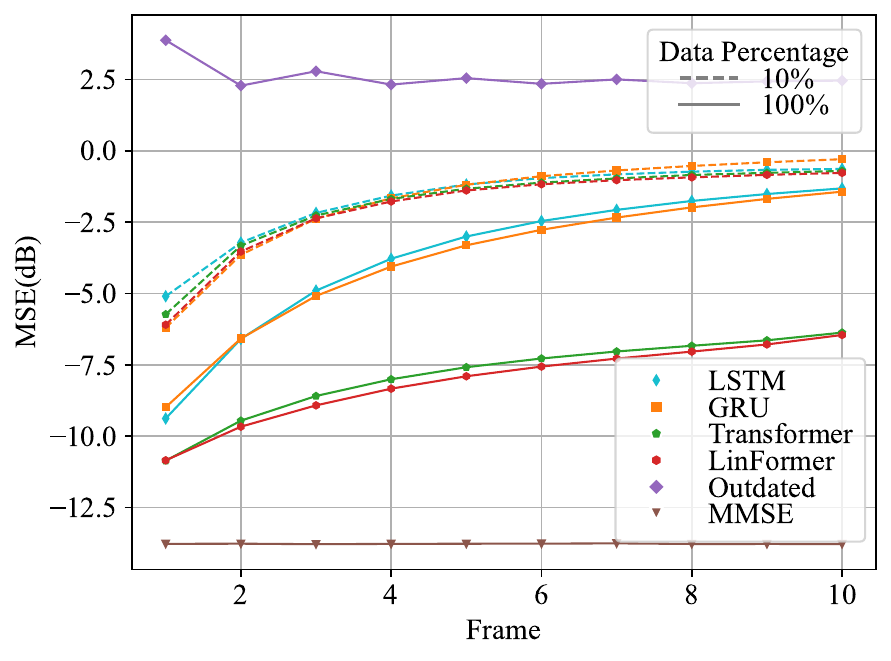}
    } \\ 
  \subfloat[MRT channel capacity using the predicted CSI]{
      \label{fig:t_MRT_data}
      \includegraphics[width=0.4\textwidth]{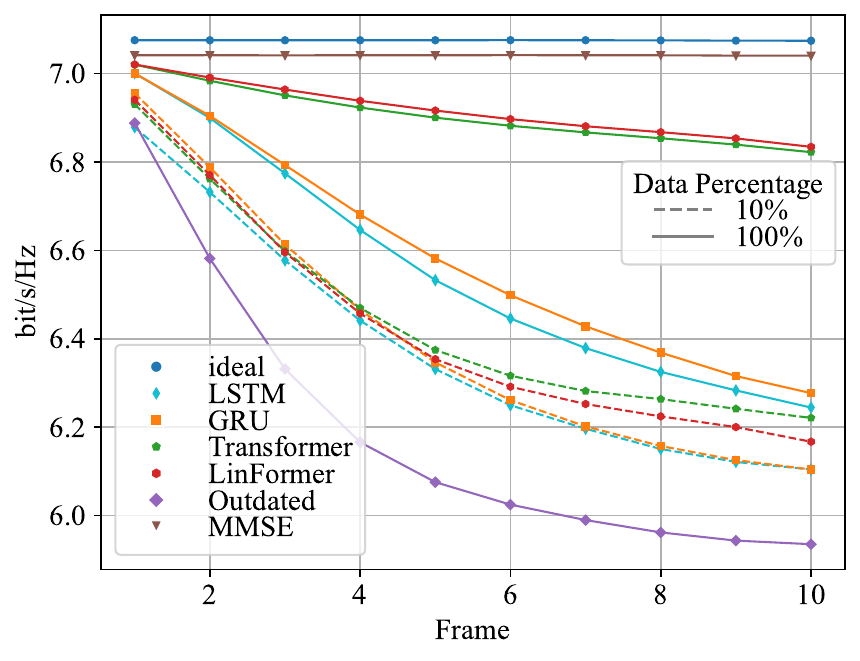}
    }
  \caption{Comparison of model performance trained using different volumes of training data.}
  \label{fig:data_percent}
\end{figure}

As shown in Fig. \ref{fig:t_MSEdB_data}, the MSE performance of all models decreases as the amount of training data increases. The LinFormer model has the best performance.
With a limited 10\% amount of training data, we achieved results similar to those in \cite{stenhammar2024comparison}, with GRU performing the best in the nearest future time steps. However, as the amount of training data increased, both Transformer and LinFormer showed significant improvements, whereas GRU and LSTM demonstrate the much smaller improvement.

As shown in Fig. \ref{fig:t_MRT_data}, the MRT channel capacity of all models increases as the amount of training data increases. We also observed that when the MSE is relatively high, the channel capacity with MRT beamforming is not inversely correlated with MSE. This occurs because, with the MRT beamforming method, the predicted channel matrix is summed over the transmit antenna dimension $T$.

\begin{figure}[htb]
  \centering
  \subfloat[Speed from from $30$ to $60$ km/h]{
      \label{fig:Speed30_60}
      \includegraphics[width=0.4\textwidth]{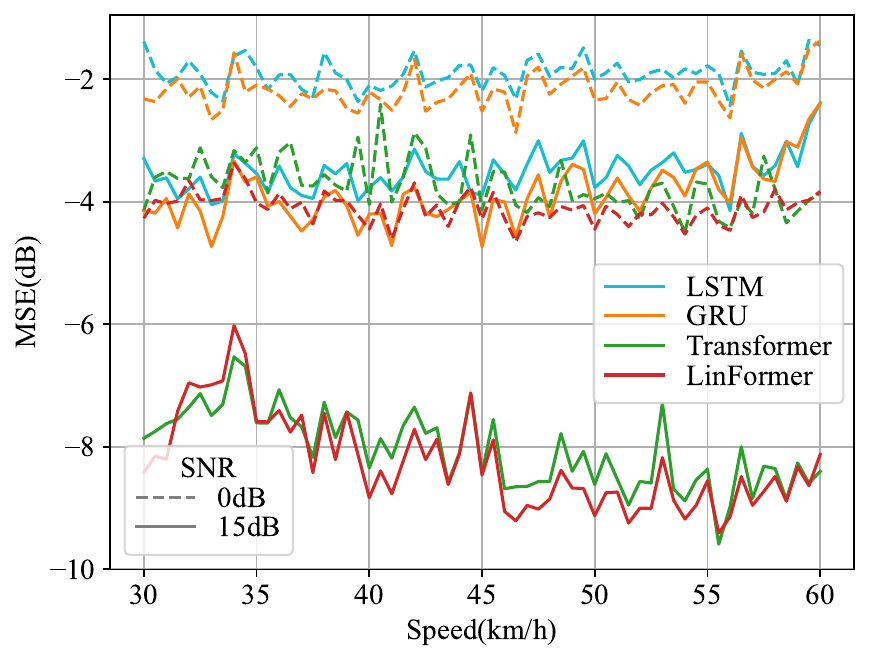}
    }\
  \subfloat[Speed from from $0$ to $300$ km/h]{
      \label{fig:Speed0_300}
      \includegraphics[width=0.4\textwidth]{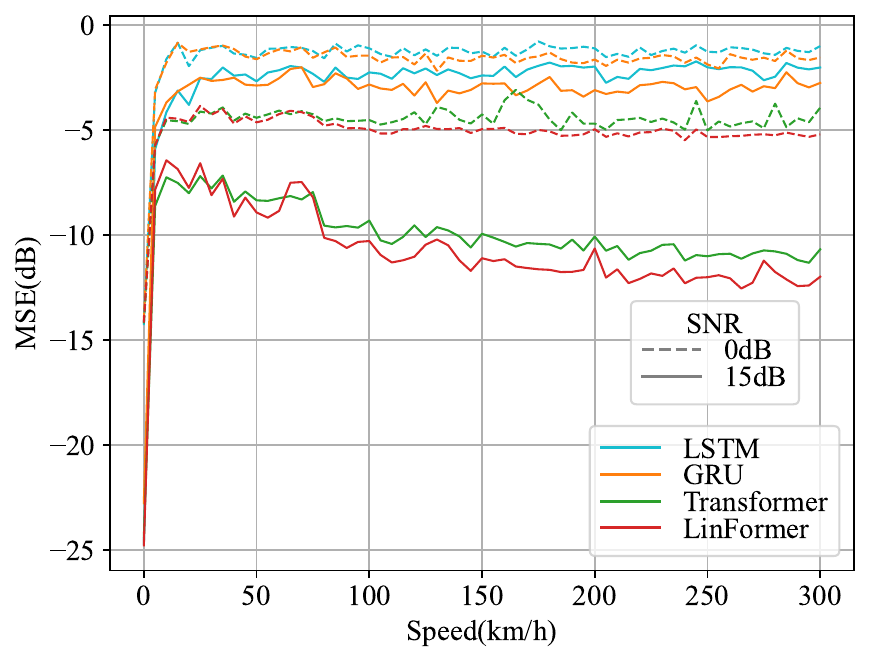}
    }
  \caption{Comparison of model performance with various speed range.}
  \label{fig:speed_com}
\end{figure}

\subsubsection{Performance Vs. Speed}

As illustrated in Fig. \ref{fig:speed_com}, and \ref{fig:Speed0_300}, we tested all models at SNR = 0 dB and 15 dB across speeds ranging from $30$ to $60$ km/h in Fig. \ref{fig:Speed30_60} and $0$ to $300$ km/h in Fig. \ref{fig:Speed0_300}, respectively. Additionally, we observe that under speed settings ranging from $0$ to $300$ km/h, the overall MSE performance surpasses that of speeds from $30$ to $60$ km/h. This improvement can be attributed to the training dataset containing approximately 6 million samples at speeds from 0 to 300 km/h, which is ten times greater than those available for speeds from 30 to 60 km/h. The richer training dataset enables the model to learn more generalized features, thereby achieving enhanced performance.

\begin{figure}[htb]
  \centering
  \includegraphics[width=0.35\textwidth]{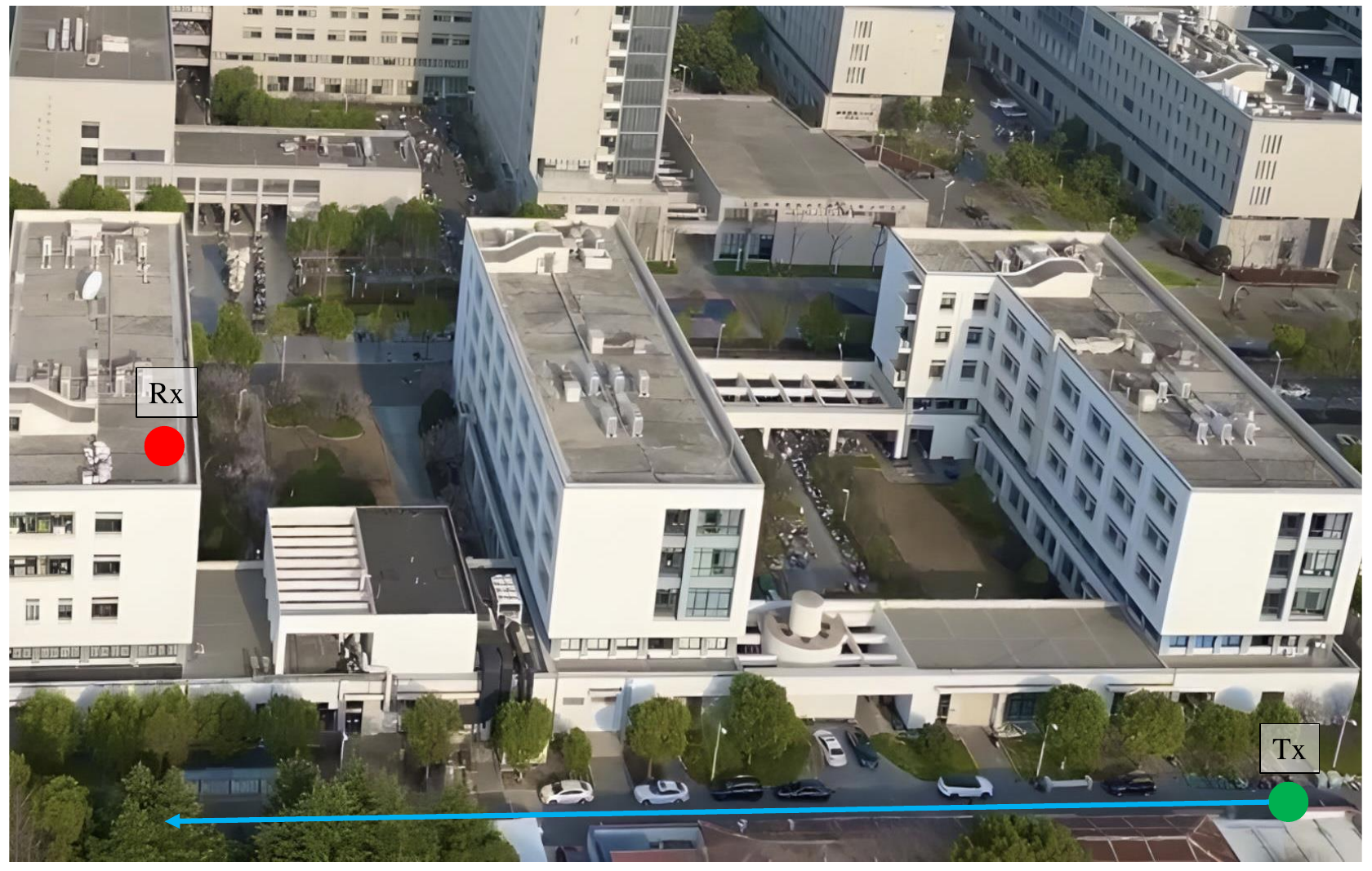}
  \caption{Illustration of CSI measurement: a vehicle (blue) moving towards a road side unit (red) on campus \cite{peng2022novel}.}
  \label{fig:SHU}
\end{figure}

\begin{table}[htbp]
    \caption{Measurement Configurations}
    \centering
    \begin{tabular}{|c|c|}
        \hline
        \textbf{Configuration} & \textbf{Value} \\ \hline
        Antenna & \makecell[c]{Omnidirectional\\ Vertical Polarization} \\ \hline
        Transmitted Sequence & Pseudo-Noise Code \\ \hline
        Sample Rate (MHz) & 200 \\ \hline
        Number of Samples in Each Symbol & 1023 \\ \hline
        Synchronization Mode & Rubidium Clock \\ \hline
        Observation Window (s) & 1.44 \\ \hline
        Prediction Window (s) & 0.16 \\ \hline
    \end{tabular}
    \label{tab:Configurations}
\end{table}

\setcounter{figure}{19}
\begin{figure*}[htb]
  \centering
  \subfloat[LSTM]{
      \label{fig:H_0_LSTM}
      \includegraphics[width=0.22\textwidth]{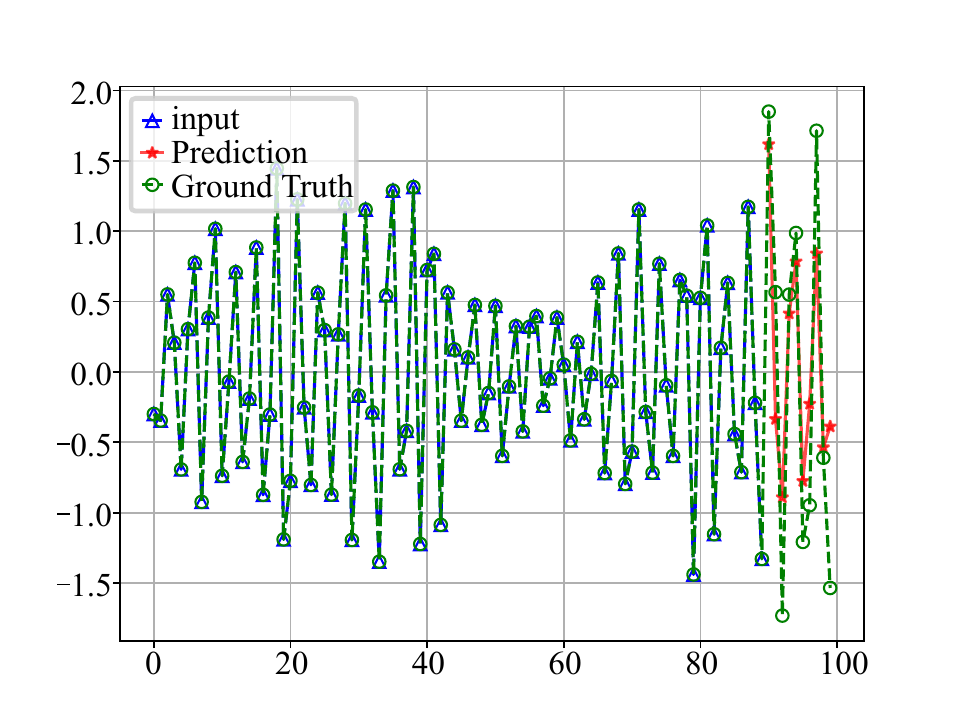}
    }
  \subfloat[GRU]{
      \label{fig:H_0_GRU}
      \includegraphics[width=0.22\textwidth]{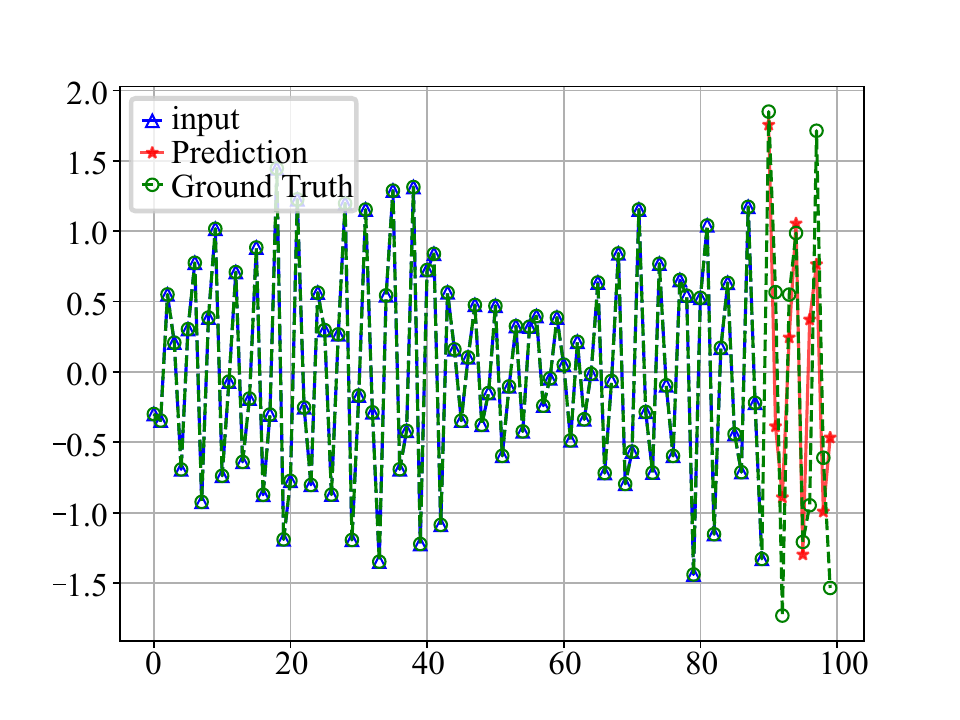}
    }
  \subfloat[Transformer]{
      \label{fig:H_0_Transformer}
      \includegraphics[width=0.22\textwidth]{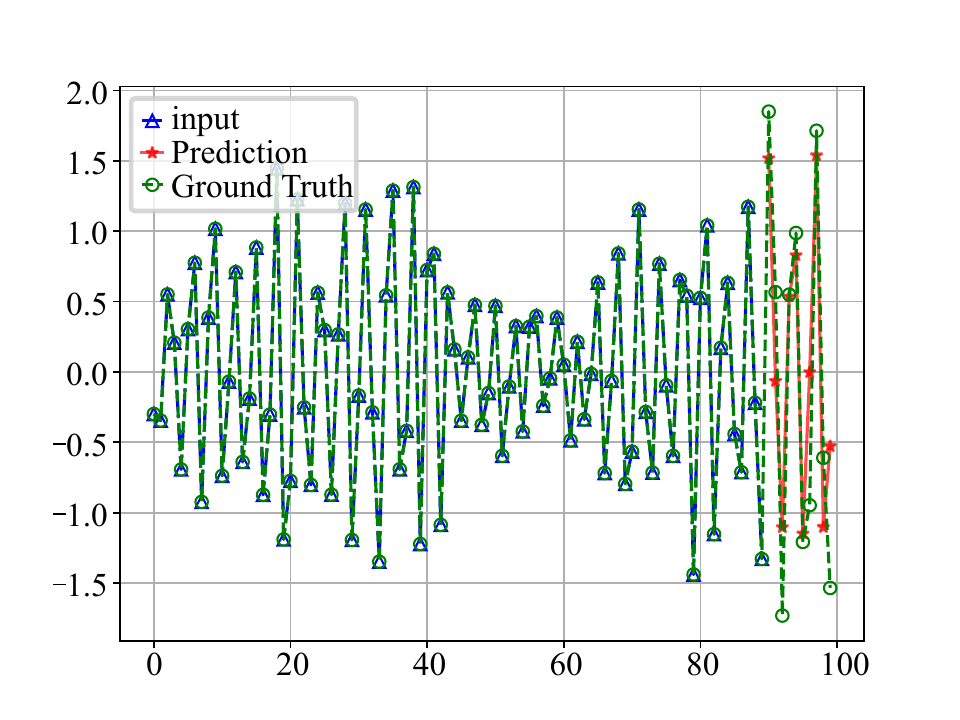}
    }
  \subfloat[LinFormer]{
      \label{fig:H_0_LinFormer}
      \includegraphics[width=0.22\textwidth]{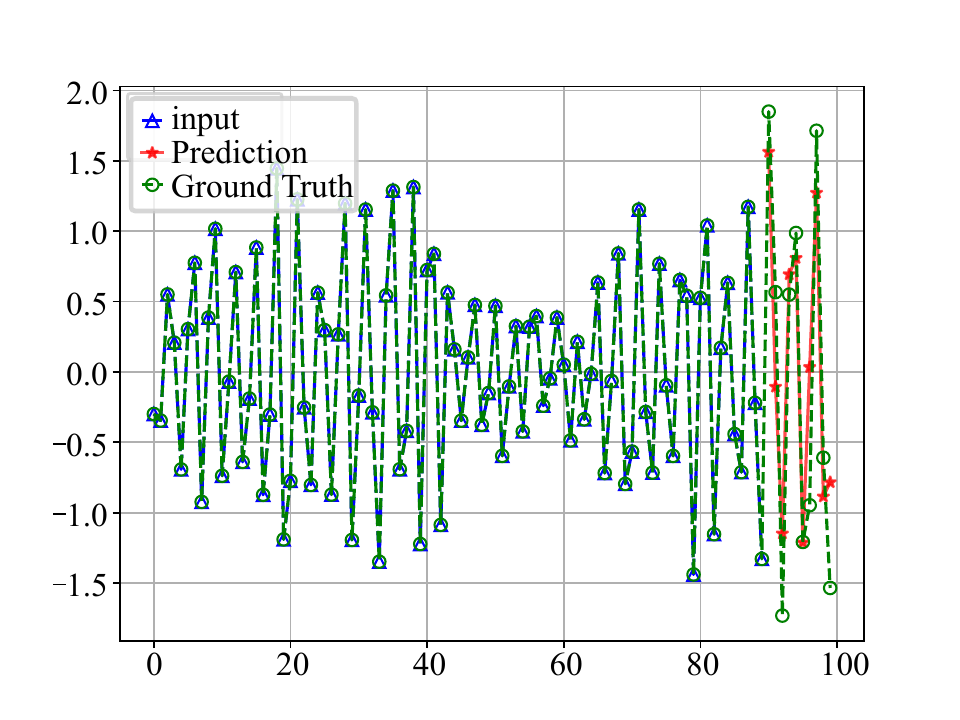}
    } \\
  \subfloat[LSTM]{
      \label{fig:H_40_LSTM}
      \includegraphics[width=0.22\textwidth]{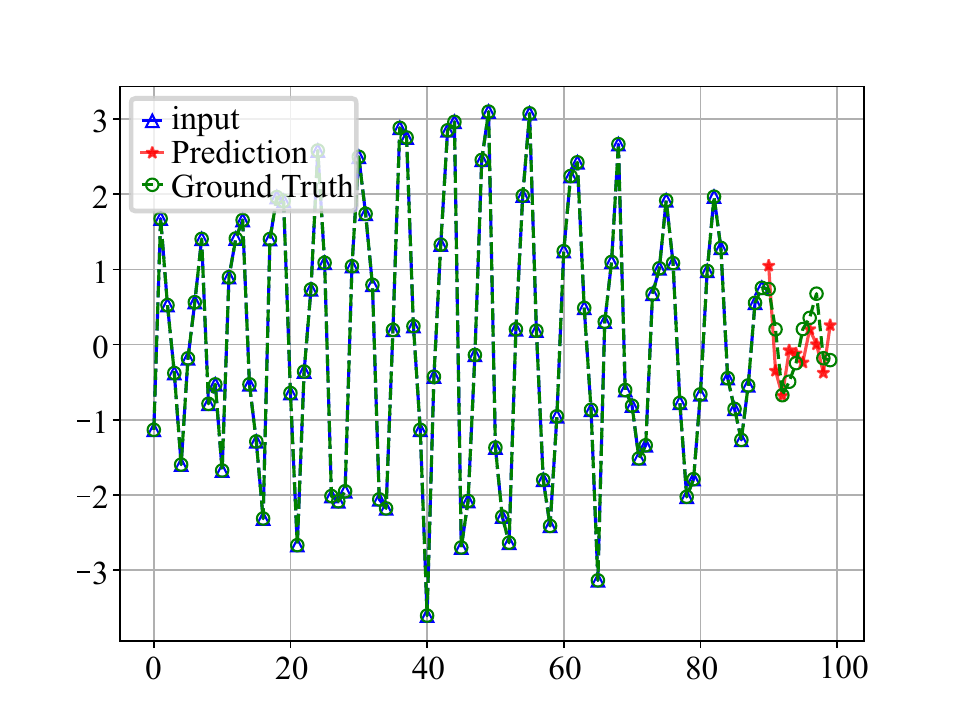}
    }
  \subfloat[GRU]{
      \label{fig:H_40_GRU}
      \includegraphics[width=0.22\textwidth]{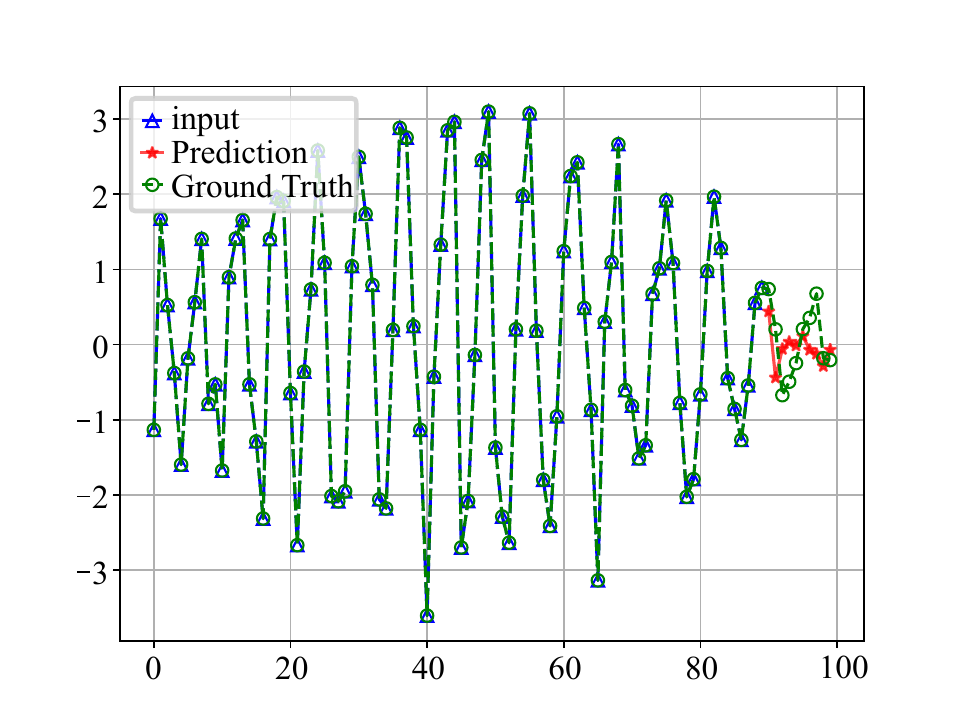}
    }
  \subfloat[Transformer]{
      \label{fig:H_40_Transformer}
      \includegraphics[width=0.22\textwidth]{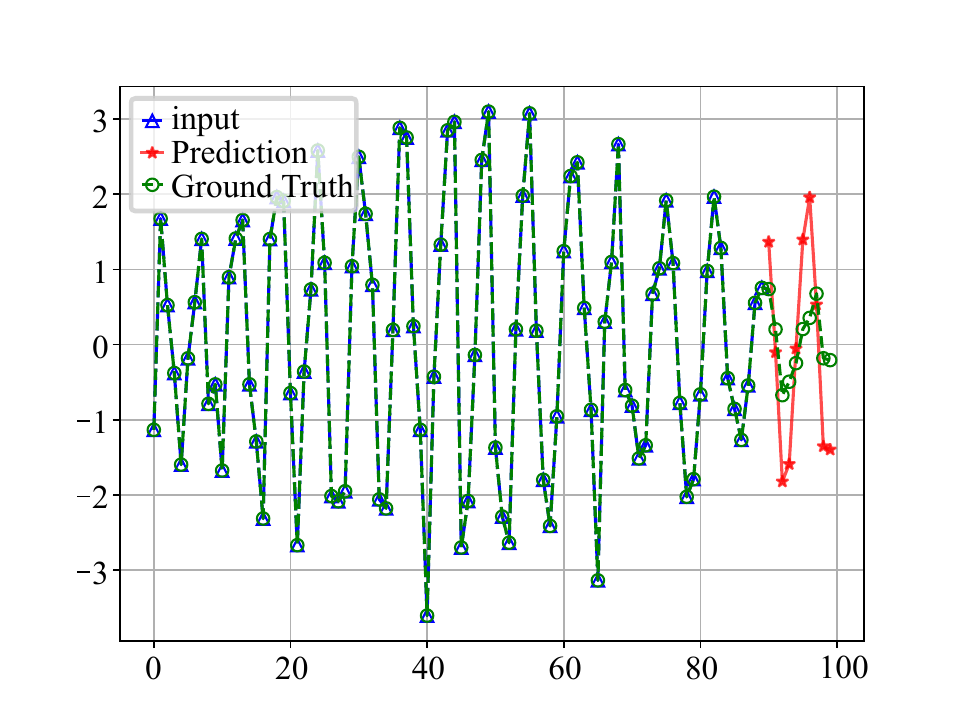}
    }
  \subfloat[LinFormer]{
      \label{fig:H_40_LinFormer}
      \includegraphics[width=0.22\textwidth]{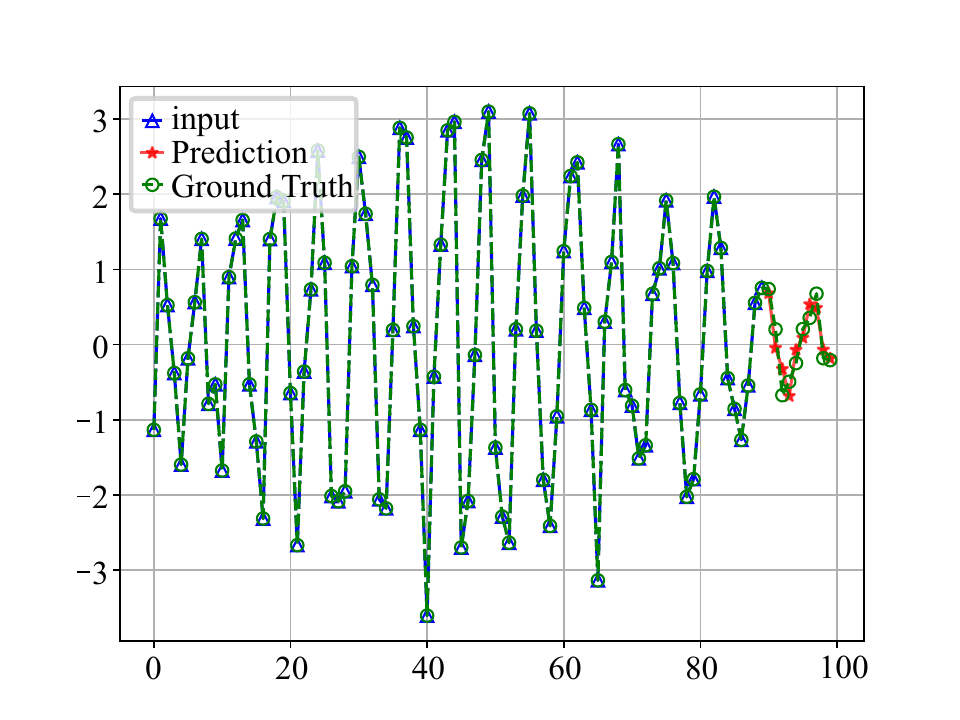}
    }
  \caption{The real part of one channel in the real-world test sample 1 (first row). The real part of one channel in the real-world test sample 41 (second row).}
  \label{fig:real_world_test}
\end{figure*}

\subsection{Generalization Analysis Using Measured CSI}
Extensive simulations using the simulated CSI data have demonstrated the superiority of the proposed LinFormer, to further evaluate the generality of the proposed LinFormer, we propose to fine-tune \cite{nie2022time} all models using measured CSI.

\subsubsection{CSI Measurement}
The CSI measurement is conducted in V2X communication scenarios \cite{peng2022novel}, where a vehicle is driving towards a road side unit as shown in Fig. \ref{fig:SHU}. The detailed measurement configurations are listed in Table \ref{tab:Configurations}. 

\setcounter{figure}{17}
\begin{figure}[htb]
  \centering
  \includegraphics[width=0.48\textwidth]{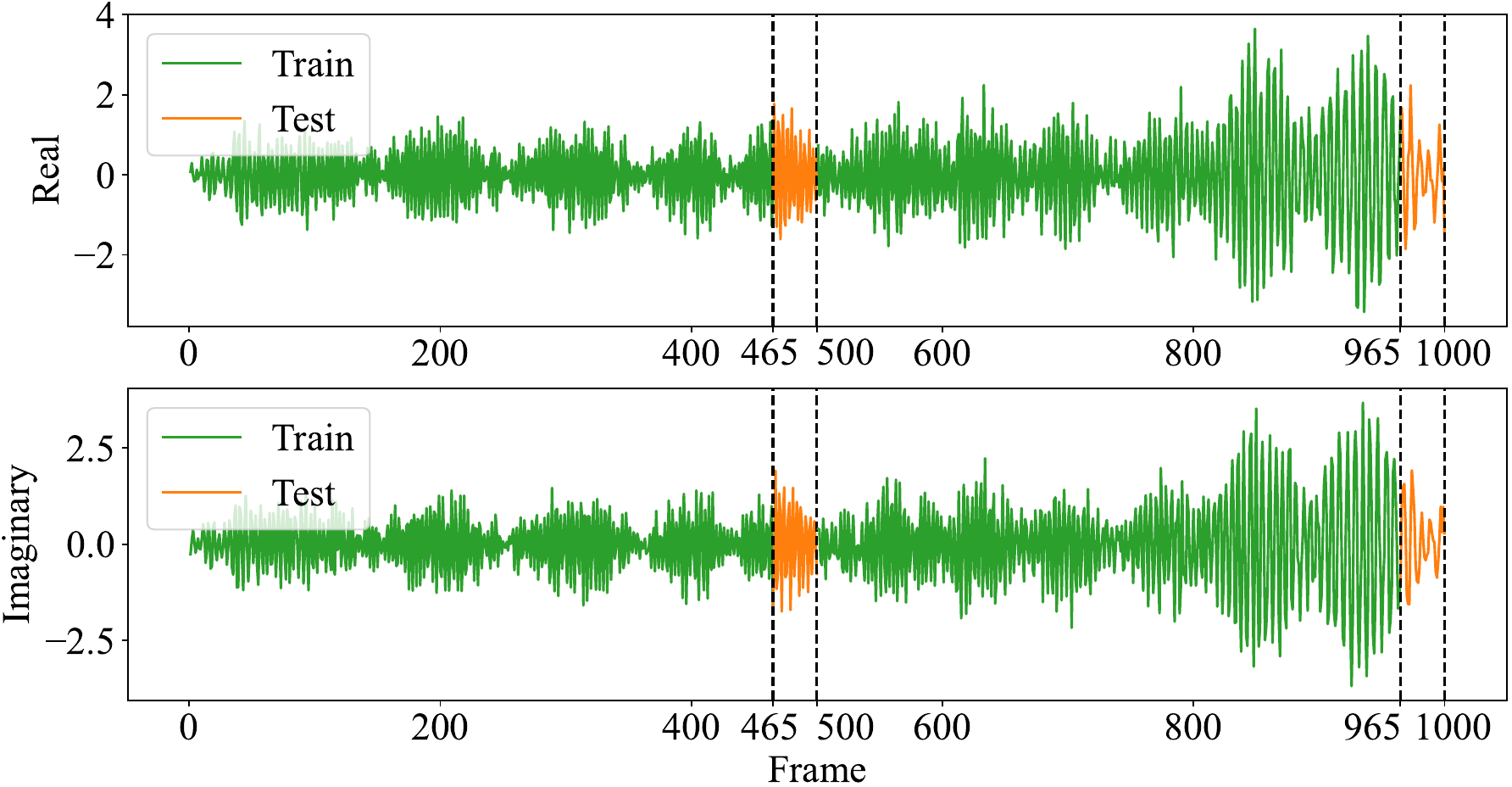}
  \caption{The real and imaginary parts of a channel.}
  \label{fig:real_data}
\end{figure}

As shown in Fig. \ref{fig:real_data}, to ensure diversity, we select the middle and end segments to get test data, ensuring that the test data is not included in the training set. Consistent with the simulation data approach described in the previous section, we apply the sliding window method to generate each training sample. The input historical channel sequence has a length of 90, while the channel sequence to be predicted has a length of 10. There are 750 training samples for fine-tuning the model and 50 samples for testing.

\begin{figure}[htb]
  \centering
  \includegraphics[width=0.4\textwidth]{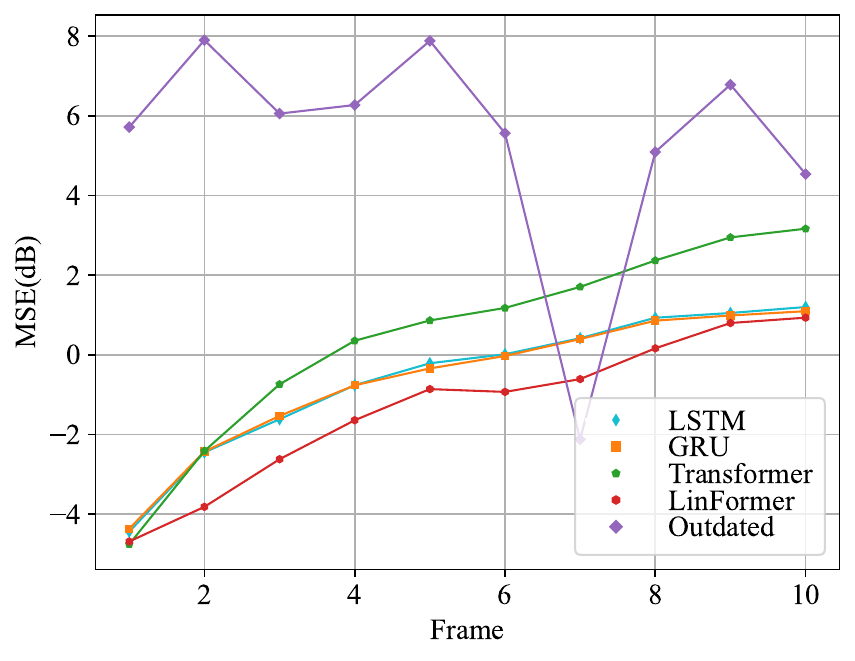}
  \caption{MSE performance of different models tested on real-world data.}
  \label{fig:real_data_MSE}
\end{figure}

\subsubsection{Performance Analysis}
We utilized 750 measured CSI training samples to fine-tune our pre-trained model, which was initially trained on a substantial amount of simulated data. 

The MSE for all test samples is shown in Figure \ref{fig:real_data_MSE}. The proposed LinFormer achieves the best performance, followed by the LSTM and GRU, while the Transformer performs the worst. This indicates that, even when trained on a large amount of simulation data, the Transformer struggles to achieve good results when fine-tuned with limited data. In contrast, the proposed LinFormer effectively overcomes this limitation, demonstrating that LinFormer is easier to train and exhibits superior generalization performance.

The channel prediction performance of different models on two test samples are illustrated in Fig. \ref{fig:real_world_test}. 
In the first row of the figure, we observe that all four models perform relatively well. While, LinFormer demonstrates superior performance on test samples from the end of the vehicle's trajectory, as shown in the second row. We attribute the superior performance of LinFormer to its generalization ability, requiring only a small number of samples to fine-tune and achieve impressive results.


\section{Conclusions}
In this paper, we propose the LinFormer model, a novel approach tailored for wireless channel prediction tasks to overcome limitations observed in Transformer models. Specifically, the LinFormer model employs an encoder-only architecture and proposes a novel TMLP module to substitute the attention mechanism in Transformers, making it more applicable for channel prediction with reduced computational complexity. We adopt a WMSE loss function and data augmentation techniques to improve the performance of the model in low SNR regions. Extensive simulations reveal that the LinFormer model demonstrates superior efficiency in leveraging the scaling law compared to standard Transformer architectures. Furthermore, when evaluated using both simulated and measured CSI data, our proposed LinFormer exhibits marked improvements in channel prediction accuracy and inference speed over existing frameworks. These results underscore the potential of the LinFormer model to significantly advance the field of wireless channel prediction, offering a more efficient and accurate alternative to current methodologies.


\ifCLASSOPTIONcaptionsoff
  \newpage
\fi

\bibliographystyle{IEEEtran}
\bibliography{bibfile}

\end{document}